\documentclass[lettersize,journal]{IEEEtran}
\usepackage[utf8]{inputenc} 
\usepackage[T1]{fontenc}    
\usepackage{hyperref}       
\usepackage{url}            
\usepackage{booktabs}       
\usepackage{amsfonts}       
\usepackage{nicefrac}       
\usepackage{microtype}      
\usepackage{xcolor}         
\usepackage{enumitem}
\usepackage{amsmath,amsthm,amssymb,amsfonts,bbm,bm}
\allowdisplaybreaks
\usepackage{algorithm,algpseudocode}
\usepackage{hyperref}

\usepackage{graphicx}
\usepackage{subfigure}
\renewcommand{\v}[1]{\ensuremath{\boldsymbol{#1}}}
\newtheorem{theorem}{Theorem}
\newtheorem{lemma}{Lemma}

\newtheorem{proposition}{Proposition}
\newtheorem{remark}{Remark}
\newtheorem{corollary}{Corollary}
\usepackage[capitalise]{cleveref}
\usepackage{stfloats}

\DeclareMathOperator*{\tsum}{\textstyle\sum}

\begin{document}

\title{Provable Privacy Advantages of Decentralized Federated Learning via Distributed Optimization}

\author{Wenrui Yu\IEEEauthorrefmark{1}, Qiongxiu Li\IEEEauthorrefmark{1}\IEEEauthorrefmark{2},~\IEEEmembership{Member, IEEE}, Milan Lopuha\"a-Zwakenberg, Mads Gr\ae sb\o ll Christensen,~\IEEEmembership{Senior Member, IEEE}, and Richard Heusdens,~\IEEEmembership{Senior Member, IEEE}
\thanks{W. Yu is with the CISPA Helmholtz Center for Information
Security, Germany.}
\thanks{Q. Li and M.G. Christensen are with Aalborg University, Denmark.}
\thanks{M. Lopuha\"a-Zwakenberg is with University of Twente, the Netherlands.}
\thanks{R. Heusdens is with Netherlands Defence Academy and Delft University of Technology, the Netherlands.}
\thanks{\IEEEauthorrefmark{1}Equal contribution,~~\IEEEauthorrefmark{2}Corresponding author: Qiongxiu Li.}
\thanks{Emails: wenrui.yu@cispa.de; qili@es.aau.dk; m.a.lopuhaa@utwente.nl; mgc@es.aau.dk; r.heusdens@tudelft.nl}}



\maketitle
\raggedbottom
\addtolength{\abovedisplayskip}{-1.0mm}
\addtolength{\belowdisplayskip}{-1.0mm}

\begin{abstract}
Federated learning (FL) emerged as a paradigm designed to improve data privacy by enabling data to reside at its source, thus embedding privacy as a core consideration in FL architectures, whether centralized or decentralized. Contrasting with recent findings by Pasquini et al., which suggest that decentralized FL does not empirically offer any additional privacy or security benefits over centralized models, our study provides compelling evidence to the contrary. We demonstrate that decentralized FL, when deploying distributed optimization, provides enhanced privacy protection - both theoretically and empirically - compared to centralized approaches.  The challenge of quantifying privacy loss through iterative processes has traditionally constrained the theoretical exploration of FL protocols.  We overcome this by conducting a pioneering in-depth information-theoretical privacy analysis for both frameworks.  Our analysis, considering both eavesdropping and passive adversary models, successfully establishes bounds on privacy leakage. In particular, we show information theoretically that the privacy loss in decentralized FL is upper bounded by the loss in centralized FL. Compared to the centralized case where local gradients of individual participants are directly revealed,  a key distinction of optimization-based decentralized FL is that the relevant information includes differences of local gradients over successive iterations and the aggregated sum of different nodes' gradients over the network. This information complicates the adversary's attempt to infer private data. To bridge our theoretical insights with practical applications,  we present detailed case studies involving logistic regression and deep neural networks. These examples demonstrate that while privacy leakage remains comparable in simpler models, complex models like deep neural networks exhibit lower privacy risks under decentralized FL. Extensive numerical tests further validate that decentralized FL is more resistant to privacy attacks,  aligning with our theoretical findings.
\end{abstract}


\section{Introduction}
\label{sec:intro}
Federated Learning (FL) enables collaborative model training across multiple participants/nodes/clients without directly sharing each node's raw data~\cite{mcmahan2017communication}. FL can operate on either a centralized/star topology or a decentralized topology, as shown in \autoref{fig:topo}~\cite{li2020federated}. The prevalent centralized topology requires a central server that interacts with each and every node individually.
The main procedure of a centralized FL protocol typically unfolds in three steps: 1) Nodes train local models based on their own private dataset and transmit model updates, such as gradients, to the server; 2) The server aggregates the local models to a global model and redistributes to the nodes; 3) Nodes update the local models based on the global model and send the model updates back to the server. The process is iteratively repeated until convergence. 
However, a centralized server is not always feasible due to its high communication demands and the need for universal trust from all nodes. In addition, it poses a risk of a single point of failure, making the network vulnerable to targeted attacks. As an alternative,  decentralized FL circumvents these issues by facilitating direct data exchanges between (locally) connected nodes, thereby eliminating the need for a central server for model aggregation.  

Decentralized FL protocols, also known as peer-to-peer learning protocols, fall into two main categories. The first involves average-consensus-based protocols. With these protocols, instead of sending model parameters to a central server, nodes collaborate together to perform model aggregation nodes in a distributed manner. The aggregation is typically done by partially averaging the local updates within a node's neighborhood.  Examples of these protocols are the empirical methods where the aggregation is done using average consensus techniques such as gossiping SGD~\cite{jin2016scale}, D-PSGD~\cite{lian2017can}, and variations thereof~\cite{tang2018d,hu2019decentralized}. 
 The second category comprises protocols that are based on distributed optimization, referred to as optimization-based decentralized FL. These (iterative) methods 
directly formulate the underlying problem as a constrained optimization problem and employ distributed solvers like ADMM~\cite{mota2013d,li2019communication,chen2021coded} or PDMM~\cite{zhang2017distributed,sherson2018derivation,niwa2020edge} to solve them. The constraints are formulated in such a way that, upon convergence, the learned models at all nodes are identical. Hence, there is no explicit separation between updating local models and the update of the global model, i.e., the three steps in centralized FL mentioned before are executed simultaneously. 

Despite not directly sharing private data with servers or nodes,  FL is shown vulnerable to privacy attacks as the exchanged information, such as gradients or weights, still poses a risk for privacy leakage.
Existing work on privacy leakage predominantly focuses on the centralized case. 
A notable example is the gradient inversion attack~\cite{zhu2019deep,zhao2020idlg,geiping2020inverting,yin2021see,boenisch2021curious,geng2023improved,wei2020framework,yang2022using,zhao2022deep,xu2022agic}, an iterative method for finding input data that produce a gradient similar to the gradient generated by the private data. Such attacks are based on the assumption that similar gradients are produced by similar data samples.  
\begin{figure}[t]
    \centering
    \includegraphics[width=0.40\textwidth]{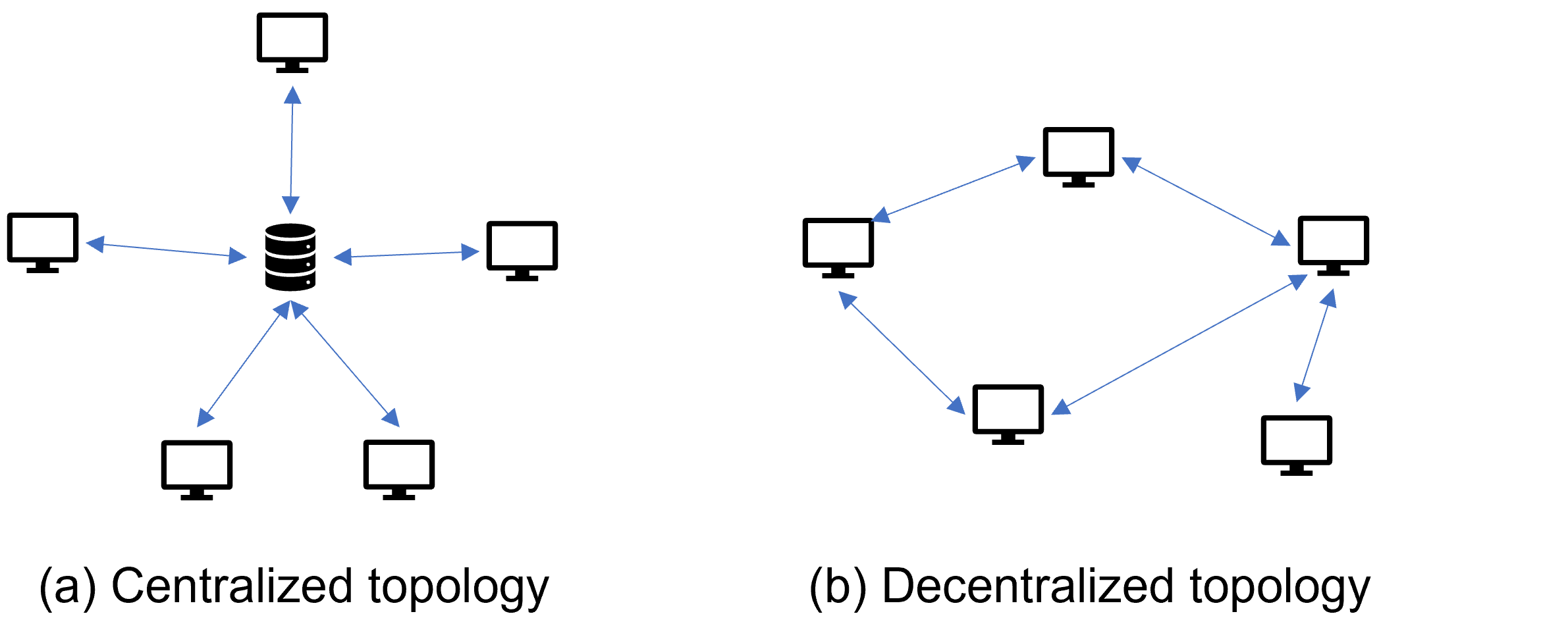}
    \caption{Two topologies in federated learning}
    \label{fig:topo}
\end{figure}

In exploring the privacy aspects of centralized versus decentralized FL, many works claim that the decentralized FL is more privacy-preserving than centralized FL without any privacy argument~\cite{yuan2016convergence,cheng2019towards,vogels2021relaysum}. The main idea is that sensitive information, such as private data, model weights, and user states, can no longer be observed or controlled through a single server. However, recent empirical findings challenge such a claim, particularly for average-consensus-based decentralized FL protocols. It is shown in ~\cite{pasquini2022privacy} that these protocols may not inherently offer better privacy protections over centralized FL, and might even increase susceptibility to privacy breaches. For instance, it has been shown that an arbitrary colluding client could potentially obtain the same amount of information as a central server in centralized FL when inverting input private data via gradients. In contrast, the privacy implications of optimization-based decentralized FL protocols have, to the best of our knowledge, been rarely investigated.  This research gap is partly due to the fact that analytically tracking the privacy leakage in distributed algorithms, particularly over multiple iterations, is very challenging. The main difficulty lies in distinguishing and comprehending how information is correlated between these iterations. 
\subsection{Paper contribution}
In this paper, we take the first step to perform a theoretical privacy analysis of both centralized and decentralized FL frameworks by analyzing the information flow within the network. Our key contributions are summarized below:
\begin{itemize}[itemsep=2pt,topsep=0pt,parsep=0pt]
    \item \emph{Analytical privacy bounds of decentralized FL}:  We conduct an information-theoretical privacy analysis of both centralized and decentralized FL protocols,  using mutual information as a key metric.   To the best of our knowledge, this is the first information-theoretical privacy analysis in this context.   Notably,  we derive two privacy bounds for the optimization-based decentralized FL and show that its privacy loss is upper bounded by the loss of centralized FL (Theorem \ref{thm_pas2} in Section \ref{sec:bound}).  We further exemplify the derived privacy gap through two applications, including logistic regression and Deep Neural Networks (DNNs).
    \item \emph{Empirical validation through privacy attacks}:
For DNN applications, we show that in the case of optimization-based decentralized FL, gradient inversion attacks can be applied to reconstruct the original input data, but the reconstruction performance is degraded compared to centralized FL due to the limited amount of information available to the adversary. A similar trend is also observed when evaluated using membership inference attacks. 
Overall, decentralized FL employing distributed optimization is shown to be less vulnerable to privacy attacks compared to centralized FL, consistent with our theoretical findings.  This finding challenges the previous belief that decentralized FL offers no privacy advantages compared to centralized FL~\cite{pasquini2022privacy} (see Section \ref{ssec.epfl} for a detailed explanation). 
\end{itemize}
\subsection{Outline and Notation}
The paper is organized as follows. Section \ref{sec:prelim} reviews necessary fundamentals, and Section \ref{sec.priMetric} introduces the involved metrics for quantifying privacy.  Section \ref{sec:prop} introduces the optimization-based decentralized FL protocol.  Section \ref{sec:bound} analyzes the privacy of both centralized and decentralized FL protocols and states the main result. 
Section \ref{sec:experimentI} analyzes logistic regression example. Section \ref{sec:experimentII}, \ref{sec:DNNII}
analyze the application of DNNs. Conclusions are given in Section \ref{sec:conclusion}. 

We use bold lowercase letters to denote vectors $\v x$ and bold uppercase letters for matrices $\v X$. Calligraphic letters $\mathcal{X}$ denotes sets. The $i$th entry of a vector $\v x$ is denoted $\v x_i$.  The superscript $(\cdot)^\intercal$ denotes matrix transposition. $\v I$ is used to denote the identity matrix of appropriate dimensions. $\v 0$ and $\v 1$ are the all-zero and all-one vectors. $\nabla$ denotes the gradient. The value of the variable $\v x$ at iteration $t$ is denoted as  $\v x^{(t)}$. 
We use $\|\cdot\|$  to indicate the $\ell_2$-norm and $\mathrm{ran}(\cdot)$ and $\mathrm{ker}(\cdot)$ to denote the range and kernel of their argument, respectively.
For the sake of notational simplicity, we represent random variables using capital letters,  regardless of whether its outcome is a scalar, a vector, or a matrix. 
$M\left(\boldsymbol{x}, i\right):=\mathbbm{1}\{(\cdot, \v x)\in \mathcal{D}\}$ is the indicator variable showing whether $\v x$ is in the dataset $\mathcal{D}$.

\section{Preliminaries}
\label{sec:prelim}
This section reviews the necessary fundamentals for the remainder of the paper. 
\subsection{Centralized FL}
Without loss of generality, we focus on classification problems involving $n$ nodes, each with its local dataset $\{(\v x_{ik}, \ell_{ik}): k=1,\ldots,n_i\}$,  where $\v x_{ik} \in \mathbb{R}^v$ represents an input sample,  $\ell_{ik} \in \mathbb{R}$ is the associated label and $n_i$ is the number of input samples. The dimension $v$ of the data samples is application-dependent. Collecting the $\v x_{ik}$s and $\ell_{ik}$s, we define $\v x_i=(\v x^\intercal_{i1},\ldots,\v x^\intercal_{i n_i})^\intercal$ and $\v \ell_i=(\ell_{i1},\ldots,\ell_{i n_i})^\intercal$.  Let $f_i(\v w_i,(\v x_i,\v \ell_i))$ denote the cost function of node $i$ where $\v w_i\in\mathbb{R}^u$ is the model weight to be learned from the input dataset $(\v x_i,\v \ell_i)$, whose dimension, again, depends on the application. In the remainder of the paper we will omit the $(\v x_i,\v \ell_i)$ dependency for notational convenience when it is clear from the context and simply write $f_i(\v w_i)$.
A typical centralized FL protocol works as follows: 
\begin{enumerate}
    \item Initialization: at iteration  $t=0$, the central server randomly initializes the weights $\v w_i^{(0)}$ for each node.
    \item Local model training: at each iteration $t$, each user $i$ first receives the model updates from the server and then computes its local gradient, denoted as $\nabla f_i(\v w_i^{(t)})$, using its local data $\v x_i$.
    \item Model aggregation: the server collects these local gradients and performs aggregation to update the global model. The aggregation is often done by weighted averaging and typically uniform weights are applied, i.e., $\frac{1}{n} \tsum_{i=1}^{n} \nabla f_i(\v w_i^{(t)})$. Subsequently,  each node $i$ then updates its own model weight  by
\begin{align}\label{eq.w_ave}
   \v  w_i^{(t+1)}=\v w_i^{(t)}- \frac{\mu}{n} \tsum_{i=1}^{n} \nabla f_i(\v w_i^{(t)}),
\end{align}
where $\mu$ is a constant controlling the convergence rate. 
The last two steps are repeated until the global model converges or until a predetermined stopping criterion is reached.
\end{enumerate}
This algorithm is often referred to as the FedAvg~\cite{mcmahan2017communication}.

\subsection{Decentralized FL}\label{ssec.dfl}
Decentralized FL works for cases where a trusted centralized server is not available. In such cases, it works on a so-called distributed network which is often modeled as an undirected graph: $\mathcal{G}=(\mathcal{V},\mathcal{E})$, with $\mathcal{V}={\{1,2,...,n}\}$ representing the node set and $\mathcal{E}\subseteq \mathcal{V}\times \mathcal{V}$ representing the edge set. $\mathcal{N}_i=\{j\,|\, {(i,j)\in \mathcal{E}\}}$ denotes the set of neighboring nodes of node $i$. In this decentralized setup, each node $i$ can only communicate with its neighboring nodes $j \in {\cal N}_i$, facilitating peer-to-peer communication without any centralized coordination.  

\subsubsection{Average consensus-based approaches}\label{sssec:avgDFL}
The model aggregation step requires all nodes' local gradients. 
Many decentralized FL protocols work by deploying distributed average consensus algorithms to compute the average of local gradients, i.e., computing $\frac{1}{n} \tsum_{i=1}^{n} \nabla f_i(\v w_i^{(t)})$ in \cref{eq.w_ave} without any centralized coordination. Example average consensus algorithms are gossip ~\cite{dimakis2010gossip} and linear iterations~\cite{olshevsky2009convergence}, which allow peer-to-peer communication over distributed networks.

The common decentralized FL often works similarly to the FedAvg algorithm, except for the step of model aggregation. For instance, D-PSGD~\cite{lian2017can,koloskova2020unified} uses
gossip averaging with neighbors to implement the aggregation, i.e., 
\begin{align}\v w_i^{(t+1)}=\v w_i^{(t)}-\frac{\mu}{d_i}\tsum_{j\in \mathcal{N}_i}\nabla f_j(\v w_j^{(t)}),
\end{align}
where $d_i=|\mathcal{N}_i|$ is the degree of node $i$.

\subsubsection{Distributed optimization-based approaches} \label{sssec:optDFL}
The goal of optimization-based decentralized FL is to collaboratively learn a global model, given the local datasets $\{(\v x_i, \v \ell_i): i\in \cal V\}$, without any centralized coordination. The underlying problem  can be posed as a constrained optimization problem given by
\begin{align} \label{eq.pmFor}
\begin{array}{ll}
{\min \limits_{\big\{\v w_i \,:\, i\in \cal V\big\}}} &\tsum\limits_{i \in \mathcal{V}} f_i(\v w_i), \\
\text{subject to} &\forall (i,j) \in {\cal E}:{\v{B}_{i\mid j}}{{\v w}_i} + {\v{B}_{j\mid i}}{{\v w}_j} = \v 0, \rule[4mm]{0mm}{0mm}
\end{array}
\end{align}
where $\v B_{i\mid j}$ and $\v B_{i\mid j}$ define linear edge constraints. To ensure all nodes share the same model at convergence (consensus constraints) we have $\v B_{i\mid j} = -\v B_{j\mid i} = \pm \v I$.
 In the following, we will use the convention that $ \v B_{i|j} = \v I$ if  $i < j$ and $ \v B_{i|j} = -\v I$.

In what follows we will refer to centralized FL as CFL. 
While decentralized FL encompasses both average consensus-based and distributed optimization-based approaches (recall Section \ref{ssec.dfl}), for simplicity, we will use the abbreviation DFL to specifically refer to the optimization-based decentralized FL as it is our main focus. We will differentiate between the two methods in contexts where such distinction is necessary to avoid confusion.

\subsubsection{Distributed optimizers}
Given the optimization problem \cref{eq.pmFor}, many distributed optimizers, notably ADMM~\cite{boyd2011distributed} and PDMM~\cite{zhang2017distributed,sherson2018derivation} have been proposed. 
From a monotone operator theory perspective~\cite{ryu2016primer,sherson2018derivation}, ADMM can be seen as a $\frac{1}{2}$-averaged version of PDMM, allowing both to be analyzed within the same theoretical framework. Due to the averaging, ADMM is generally slower than PDMM, assuming it converges. 
Both ADMM and PDMM solve the optimization problem \cref{eq.pmFor} iteratively, with the update equations for node $i$ given by:
\begin{align}
&\v w_{i}^{(t)} = \arg\min_{\v w_i} \big(
 f_{i}(\v w_{i})+ \tsum_{j \in \mathcal{N}_i} \v z_{i|j}^{(t)\intercal} \v B_{i|j}\v w_i + \frac{\rho d_i}{2}\v w_{i}^2 \big), \label{eq.xupNQ}
 \\
    &\forall j \in \mathcal{N}_i: \v z_{j|i}^{(t+1)}=(1-\theta) \v z_{j|i}^{(t)}+\theta \big(\v z_{i|j}^{(t)}+2\rho \v B_{i|j}\v w_i^{(t)}\big), \label{eq.zupNQ} 
\end{align}
where $\rho$ is a constant controlling the rate of convergence. The parameter $\theta \in (0,1]$ controls the operator averaging with $\theta=\frac{1}{2}$ (Peaceman-Rachford splitting) yielding ADMM and $\theta=1$ (Douglas-Rachford splitting) leading to PDMM. $\v z$ is called auxiliary variable having entries indicated by $\v z_{i|j}$ and $\v z_{j|i}$, held by node $i$ and $j$, respectively, related to edge $(i,j) \in \mathcal{E}$. \cref{eq.xupNQ} updates the local variables (weights) $\v w_i$, whereas \cref{eq.zupNQ} represents the exchange of data in the network through the auxiliary variables $\v z_{j|i}$. 

\subsection{Threat models}
We consider two types of adversary models: the eavesdropping and the passive (also known as honest-but-curious) adversary model.  While eavesdropping can typically be addressed through channel encryption \cite{dolev1993perfectly}, it remains a pertinent concern in our context. This relevance stems from the nature of iterative algorithms, where communication channels are utilized repeatedly, continuously encrypting each and every message incurs high communication overhead.  Therefore, in our framework, we assume that network communication generally occurs over non-secure channels, except for the initial network setup phase, details of which will be discussed later (see Section \ref{sec:prop}).
The passive adversary consists of a number of colluding nodes, referred to as corrupt nodes, which comply with the algorithm instructions but utilize the received information to infer the private input data of the other so-called honest nodes. Consequently,  the adversary has access to the following information: (a) all information gathered by the corrupt nodes, and (b) all messages transmitted over unsecured (i.e., non-encrypted) channels. 

\section{Privacy evaluation}\label{sec.priMetric}
When quantifying privacy, there are mainly two types of metrics: 1) empirical evaluation which assesses the susceptibility of the protocol against established privacy attacks, and  2) information-theoretical metrics which offer a robust theoretical framework independent of empirical attacks.
In this paper, we first evaluate privacy via an information-theoretical metric and then deploy empirical attacks to validate our theoretical results. 
\subsection{Information-theoretical privacy metric}
Among the information-theoretical metrics, popular ones include for example 
1) $\epsilon$-differential privacy \cite{dwork2006,dwork2006calibrating} which guarantees that the posterior guess of the adversary relating to the private data is only slightly better (quantified by $\epsilon$) than the prior guess; 2) mutual information~\cite{cover2012elements} which quantifies statistically how much information about the private data is revealed given the adversary's knowledge.
In this paper, we choose mutual information as the information-theoretical privacy metric. The main reasons are the following. Mutual information has been proven effective in measuring privacy losses in distributed settings~\cite{Jane2020TIFS}, and has been applied in various applications~\cite{lopuhaa2019information, Jane2020GSP,yagli2020information,bu2020tightening,li2023adaptive,liu2021quantitative,mo2020layer}. Secondly, mutual information is intrinsically linked to $\epsilon$-differential privacy (see~\cite{cuff2016differential} for more details) and is more feasible to realize in practice~\cite{gotz2011publishing,haeberlen2011differential}. 

\subsubsection{Fundamentals of mutual information}
Given two (discrete) random variables $X$ and $Y$, the mutual information $I(X;Y)$  between $X$ and $Y$ is defined as
\begin{align}\label{eq:MI_def}
   I (X;Y)=H(X) - H(X|Y),
\end{align}
where $H(X)$ represents the Shannon entropy of $X$ and $H(X|Y)$ is the conditional Shannon entropy, assuming they exist.\footnote{In cases of continuous random variables we substitute both entropies by the differential entropy, thus $I(X;Y) = h(X) - h(X|Y)$.} It follows that $ I (X;Y)=0$ when $X$ and $Y$ are independent, indicating that $Y$ carries no information about $X$. Conversely, $ I (X;Y)$ is maximal when $Y$ and $X$ share a one-to-one correspondence.

Denote $\mathcal{V}_h$ and $\mathcal{V}_c$ as the set of honest and corrupt nodes, respectively. Let $\mathcal{O}$ denote the set of information obtained by the adversary. Hence, the privacy loss, measured by the mutual information between the private data $\v x_i$ of honest node $i\in \mathcal{V}_h$ and the knowledge available to the adversary, is given by
\begin{align}\label{eq.miAdv}
     I (X_i;\mathcal{O}).
\end{align}

\subsection{Empirical evaluation via privacy attacks}
To complement our theoretical analysis,  we incorporate empirical privacy attacks to validate our findings. In machine learning, based on the nature of the disclosed private data, privacy breaches typically fall into three types: membership inference~\cite{xu2020subject,melis2019exploiting,wainakh2021user}, the property inference ~\cite{melis2019exploiting,shokri2017membership} and the input reconstruction attack~\cite{he2019model,wang2019beyond,yang2019neural,zhang2020secret}, where the revealed information is membership (whether a particular data sample belongs to the training dataset or not), properties of the input such as age and gender, and the input training data itself, respectively. Given that \cref{eq.miAdv} measures how much information about the input training data is revealed, we align our empirical evaluation by mainly focusing on input reconstruction attacks. In FL, the gradient inversion attack has been extensively studied for its effectiveness in reconstructing input samples.
\subsubsection{Gradient inversion attack} \label{subsec.labelRecov}
The gradient inversion attack typically works by iteratively refining an estimate of the private input data to align with the observed gradients generated by such data.   For each node's local dataset $(\v x_i,\v \ell_i)$, the goal of the adversary is to recover the input data $(\v x_i,\v \ell_i)$ based on the observed gradient. A typical setup is given by~\cite{zhu2019deep}:
\begin{align}\label{eq.traInv}
(\v x_i^{\prime *}, \v \ell_i^{\prime *}) =\underset{\v x_i^{\prime}, \v \ell_i^{\prime}}{\arg \min }\big\|\nabla f_i(\v w_i,(\v x'_i,\v \ell'_i)) - \nabla f_i(\v w_i,(\v x_i,\v \ell_i))\big\|^2,
\end{align}
and many variants thereof are proposed~\cite{zhao2020idlg,geiping2020inverting,
 yin2021see,boenisch2021curious,geng2023improved,wei2020framework,yang2022using,zhao2022deep,xu2022agic}. To evaluate the quality of reconstructed inputs, we use the widely adopted structural similarity index measure (SSIM)~\cite{wang2004image}  to measure the similarity between the reconstructed images and true inputs. The SSIM index ranges from $-1$ to $1$, where $\pm 1$ signifies perfect resemblance and $0$ indicates no correlation.

\textbf{Analytical label recovery via local gradient}: 
While it appears that both the input data $\v x_i$ and its label $\v \ell_i$ in \cref{eq.traInv} require reconstruction through optimization, existing work normally assumes that the label is already known. This is because the label can often be analytically inferred from the shared gradients~\cite{zhao2020idlg, wainakh2021user}.  
The main reason is as follows. Consider a classification task where the neural network has $L$ layers and is trained with cross-entropy loss. Assume $n_i=1$ for simplicity (one data sample at each node). Let $\v y = (y_1,\ldots,y_C)$ denote the outputs (logits), where $y_i$ is the score (confidence) predicted for the $i$th class. 
 With this, the cross-entropy loss over one-hot labels is given by
 \begin{equation}
 f_i(\v w_i) = -\log\big( \frac{e^{y_{\ell_i}}}{\tsum_j e^{y_j}}\big) = \log\big(\textstyle \tsum_j e^{y_j}\big) - y_{\ell_i},
 \label{eq:ce}
 \end{equation}
 where $\log(\cdot)$ denotes the natural logarithm.
 Let $\v w_{i,L,c}$ denote the weights in the output layer $L$ corresponding to output $y_c$. The gradient of $f_i(\v w_i)$ with respect to $\v w_{i,L,c}$ can then be expressed as~\cite{zhao2020idlg}:
\begin{align}\label{eq.label}
   \nabla' f_i(\v w_{i,L,c}) \triangleq \frac{\partial f_i(\v w_{i})}{\partial \v w_{i,L,c}} = \frac{\partial f_i(\v w_{i})}{\partial y_c} \frac{\partial y_c}{\partial \v w_{i,L,c}} = g_c \v a_{L-1},
\end{align}
where  $\v a_{L-1}$ is the activation at layer $L-1$ and $g_c$ is the gradient of the cross entropy \cref{eq:ce} with respect to logit $c$:
\begin{equation}
g_c = \frac{e^{y_c}}{\tsum_j e^{y_j}} - \delta_{c,\ell_i},
\end{equation}
where $\delta_{c, \ell_i}$ is the Kronecker-delta,  defined as \(\delta_{c, \ell_i} = 1\) when \(c = \ell_i\) and \(\delta_{c, \ell_i} = 0\) otherwise.  Consequently, $g_c<0$ for $c=\ell_i$ and $g_c>0$ otherwise.  Since the activation $\v a_{L-1}$ is independent of the class index $c$, the ground-truth label $\ell_i$ can be inferred from the shared gradients since $\nabla'^{\intercal} f_i(\v w_{i,L,\ell_i})\nabla' f_i(\v w_{i,L,c}) = g_{\ell_i} g_{c}\|\v a_{L-1}\|^2 < 0$ for $c\neq \ell_i$ and positive only for $c = \ell_i$. 
When dealing with $n_i>1$,  recovering labels becomes more challenging, yet feasible approaches are available. One example approach shown in~\cite{wainakh2021user} leverages the fact that the gradient magnitude is proportional to the label frequency in untrained models. 
Hence, in the centralized FL case, label information can often be deduced from the shared gradients thus improving both the efficiency and accuracy of the reconstructed input $\v x_i^{\prime *}$ when compared to the real private input $\v x_i$~\cite{zhao2020idlg}. While for the decentralized FL protocol, we will show that the label information cannot be analytically computed for certain cases, thereby inevitably decreasing both the efficiency and reconstruction quality (see details in Remark \ref{rm.graDif}). 

\subsubsection{Membership inference attack} The goal of membership inference attacks is to determine whether a specific data sample is part of the training set of a particular model. We deployed the gradient-based membership inference attacks proposed in \cite{li2022effective} which are tailored for FL and have demonstrated superior performance compared to prior approaches such as the so-called loss-based \cite{yeom2018privacy} and modified entropy-based approaches \cite{song2021systematic}. The gradient-based approach uses cosine similarity between model updates and instance gradients, expressed as follows: 
\begin{align*}
&M\left(\boldsymbol{x}^{\prime}, i\right)\\
&=\tsum_{l' \in \mathbb{R}} \mathbbm{1}\left\{\operatorname{cosim}\left(\nabla f_i(\v w_i,(\v x', l')), \nabla f_i(\v w_i,(\v x_i,\v \ell_i))\right) \geq \gamma \right\}.
\end{align*}
The latter one uses the indicator shown as
\begin{align*}
M\left(\boldsymbol{x}^{\prime}, i\right)&=\mathbbm{1}\{\left\|\nabla f_i(\v w_i,(\v x_i,\v \ell_i))\right\|_2^2 \\
&-\|\nabla f_i(\v w_i,(\v x_i,\v \ell_i))-\sum_{l' \in \mathbb{R}} \nabla f_i(\v w_i,(\v x_i,\v \ell_i))\|_2^2>0\}.
\end{align*} The corresponding results are presented in Section \ref{subsec.mia} and \ref{subsec.mia2}.

\section{DFL using distributed optimizers}\label{sec:prop}
This section introduces distributed solvers considered in this work, explains pivotal convergence properties relevant to subsequent privacy analyses, and gives details of the decentralized protocol using distributed optimization techniques.

\subsection{Differential A/PDMM}\label{subsec:subspace}
The optimality condition for \cref{eq.xupNQ} is given by\footnote{Note that ADMM can also be applied to non-differentiable problems where the optimality condition can be expressed in terms of subdifferentials: $ 0 \in \partial f_i(\v w_i^{(t)}) + \tsum_{j \in \mathcal{N}_i} \v B_{i|j} \v z_{i|j}^{(t)}+ \rho d_i  \v w_i^{(t)}$.}
\begin{align}\label{eq.partialZero}
\;\; \v 0 = \nabla f_i(\v w_i^{(t)}) + \tsum_{j \in \mathcal{N}_i} \v B_{i|j} \v z_{i|j}^{(t)}+ \rho d_i  \v w_i^{(t)}.
\end{align}
Given that the adversary can eavesdrop all communication channels,  by inspection of \cref{eq.partialZero}, transmitting the auxiliary variables $\v z_{j|i}$  would expose $\nabla f_i(\v w_i^{(t)})$, as $\v w_i^{(t)}$ can be determined from \cref{eq.zupNQ}. Encrypting $\v z_{j|i}^{(t)}$ at every iteration would address this, albeit at prohibitive computational expenses. To circumvent this, only initial values $\v z_{j|i}^{(0)}$ are securely transmitted and $\Delta \v z_{j|i}^{(t+1)} = \v z_{j|i}^{(t+1)} - \v z_{j|i}^{(t)}$ being unencrypted in subsequent iterations~\cite{jane2022elsevier,Jane2020TSP}. Consequently, upon receiving $\Delta\v z_{j|i}^{(t+1)}$, $\v z_{j|i}^{(t+1)}$ is reconstructed as
\begin{align}\label{eq.dZ}
\v z_{j|i}^{(t+1)} = \v z_{j|i}^{(t)} + \Delta \v z_{j|i}^{(t+1)}  = \tsum_{\tau=1}^{t+1} \Delta \v z_{j|i}^{(\tau)} + \v z_{j|i}^{(0)}.    
\end{align}
Let $t_{\max}$ denote the maximum number of iteration and denote $\mathcal{T} =\{0,1,\ldots, t_{\max}\}$. 
Hence, eavesdropping only uncovers  
\begin{equation} 
    \big\{\Delta \v z_{j \mid i}^{(t+1)} : (i, j) \in \mathcal{E}, t \in \mathcal{T}\big\},
    \label{eq.eavesdrop}
\end{equation}
and $\v z_{j|i}^{(t+1)}$ remains undisclosed unless  $\v z_{j|i}^{(0)}$ is known. 

\subsection{DFL using differential A/PDMM}
ADMM is guaranteed to converge to the optimal solution for arbitrary convex, closed and proper (CCP) objective functions $f_i$, whereas PDMM will converge in the case of differentiable and strongly convex functions~\cite{sherson2018derivation}. Recently, it has been shown that these solvers are also effective when applied to non-convex problems like training DNNs~\cite{niwa2020edge}.  Note that for complex non-linear applications such as training DNNs, although exact solutions of \cref{eq.xupNQ} are usually unavailable, convergence analysis of approximated solutions has been extensively investigated. For instance, it is shown in~\cite{niwa2020edge} that PDMM, using quadratic approximations,  achieves good performance for non-convex tasks such as training DNNs. Moreover, convergence guarantees with quantized variable transmissions are investigated in~\cite{jonkman2018quantisation}.

Details of DFL using differential A/PDMM solvers are summarized in Algorithm~\ref{alg:pdmm}. Note that at the initialization it requires that each node randomly initialize $z_{i|j}^{(0)}$ from independent distributions having variance $\sigma^2_Z$ and sends it to neighbor $j\in \mathcal{N}_i$ via secure channels, also referred to as the subspace perturbation technique~\cite{Jane2020ICASSP,Jane2020TSP}. The core concept involves introducing noise into the auxiliary variable $\boldsymbol{z}$ to obscure private data from potential exposure, while the convergence of $\v w$ is not affected.
To explain this idea, consider \cref{eq.zupNQ} in a compact form:
\begin{align}
  \v z^{(t+1)}&=(1 -\theta) \v z^{(t)} +\theta\left(\v P\v z^{(t)}+2c\v P\v C \v w^{(t)}\right), \label{eq:zupc}
\end{align}
where $\v C=[\v B_{  +}^{\top},\v  B_{  -}^{\top}]^{\top}$ and $\v  B_{  +}$ and $\v  B_{  -}$ contains the positive and negative entries of $\v B$, respectively. Additionally, $\v P$ is a permutation matrix that interchanges the upper half rows and lower half rows of the matrix it multiplies, leading to $\v P\v C=[ \v B_{  -}^{\top},\v B_{  +}^{\top}]^{\top}$. 
Denote $\Psi = \mathrm{ran}(\v C) + \mathrm{ran}(\v P\v C)$, its orthogonal complement is denoted by $\Psi^{\perp}= \ker(\v C^{\top}) \cap \ker((\v P\v C)^{\top})$. Let $\v\Pi_\Psi$ represent the orthogonal projection. We can then decompose $ \v z$ into components within $\Psi$ and $\Psi^{\perp}$ as
$\v z^{(t)}=\v z_{  \Psi}^{(t)}+\v z_{  \Psi^\perp}^{(t)}$.
Note that the component $\v z_{ \Psi^\perp}^{(t)}$ is not null, requiring that the number of edges should be no smaller than the number of nodes. This condition is, however, not met in CFL with a star topology. Thus even though we deploy ADMM or PDMM for CFL, it would not give any privacy benefit.

It has been proven in~\cite{jane2022elsevier} that 
\begin{align*}
\v z_{  \Psi^\perp}^{(t)} = \frac{1}{2}\left(\v z_{  \Psi^\perp}^{(0)} + \v P\v z_{  \Psi^\perp}^{(0)}\right) + \frac{1}{2}(1-2\theta)^t\left(\v z_{  \Psi^\perp}^{(0)} - \v P\v z_{  \Psi^\perp}^{(0)}\right).
\end{align*}
Thus, for a given graph structure and $\theta$, $\v z_{  \Psi^\perp}^{(t)}$  depends solely on the initialization of the  auxiliary variable $\v z^{(0)}$. Consequently, if  $\v z_{i|j}^{(0)}$ is not known by the adversary, so does  $\v z_{i|j}^{(t)}$ for subsequent iterations.   This is key that privacy advantages can be provided when compared to centralized FL (in Remark \ref{rm.secChannel} we will analyze the privacy loss for CFL when applying a similar trick).  

\begin{algorithm}[t]
  \caption{Decentralized FL via A/PDMM}
  \label{alg:pdmm}
  \begin{algorithmic}
      \State Each node $i$ randomly initializes $z_{i|j}^{(0)}$ from independent distributions having variance $\sigma^2_Z$ and sends to neighbor $j\in \mathcal{N}_i$ via secure channels. 
          \For{$t=0,1,...$} 
            \For{each node $i \in \mathcal{V}$ \textbf{in parallel}} 
                \State $\v w_i^{(t)}=$
                \vspace{-.3\baselineskip}
                \State $\displaystyle \quad\arg \min _{\v w_i}\big(f_i\big(\v w_i\big)+\tsum_{j \in \mathcal{N}_i} \v z_{i \mid j}^{(t) \boldsymbol{\top}} \v B_{i \mid j} \v w_i+\frac{\rho d_i}{2} \v w_i^2\big)$
                \For{each $j \in \mathcal{N}_i$}
                    \State $\v z_{j \mid i}^{(t+1)}=(1-\theta) \v z_{j \mid i}^{(t)}+\theta\big(\v z_{i \mid j}^{(t)}+2 \rho \v B_{i \mid j} \v w_i^{(t)}\big)$
                    \State  $\Delta \v z_{j \mid i}^{(t+1)}=\v z_{j \mid i}^{(t+1)}-\v z_{j \mid i}^{(t)}$
                \EndFor
            \EndFor
            \For{each $i \in \mathcal{V}$, $j \in \mathcal{N}_i$}
                \State $\text{Node}_j \leftarrow \text{Node}_i(\Delta \v z_{j \mid i}^{(t+1)})$
            \EndFor
           \For{each $i \in \mathcal{V}$, $j \in \mathcal{N}_i$}
                \State $\v z_{j \mid i}^{(t+1)}= \v z_{j \mid i}^{(t)} + \Delta \v z_{j \mid i}^{(t+1)}$
            \EndFor
      \EndFor
  \end{algorithmic}
\end{algorithm}

\section{Privacy analysis} \label{sec:bound} 
In this section, we conduct the comparative analysis of privacy loss in both CFL and DFL protocols, specifically focusing on the FedAvg algorithm and the decentralized approach introduced in Algorithm~\ref{alg:pdmm}. For simplicity, we will primarily consider the case $\theta=1$, i.e., PDMM, but the results can be readily extended to arbitrary $\theta\in(0,1]$.

\subsection{Privacy loss of CFL}
In CFL, the transmitted messages include the initial model weights $\v w_j^{(0)}$, and the local gradients $\nabla f_j(\v w_j^{(t)})$ at all iterations $t\in \mathcal{T}$ of all nodes $j\in{\cal V}$.  Hence, by inspection of \cref{eq.w_ave}, we conclude that knowledge of local gradients and initial weights $\v w _j^{(0)}$ is sufficient to compute all updated model weights $\v w_j^{(t)}$ at every $t\in \mathcal{T}$. 
Hence, the eavesdropping adversary has the following knowledge
\begin{align}\label{eq.knowCFLe}
    \{\nabla f_j(\v w_j^{(t)}),\v w_j^{(t)}\}_{j\in \mathcal{V},t\in \mathcal{T}}.
\end{align}
The passive adversary, on the other hand, has the following knowledge
\begin{align}\label{eq.knowCFLp}
    \{\v x_j,\v w_j^{(t)},\nabla f_j(\v w_j^{(t)})\}_{j\in \mathcal{V}_c, t \in \mathcal{T}}.
\end{align}
Combining both sets, the privacy loss, quantified  by the mutual information between the private data $\v x_i$ and the knowledge available to the adversary (as in \cref{eq.miAdv}), is given by
\begin{align} \label{eq.miCFLp}
& I (X_i;\mathcal{O}_{\rm {CFL}})\nonumber\\
&=
I( X_i; \{X_j\}_{j\in \mathcal{V}_c},\{\nabla f_j(W_j^{(t)}),W_j^{(t)}\}_{j\in \mathcal{V},t\in \mathcal{T}}).
\end{align}

\begin{remark} \label{rm.secChannel}
We could securely transmit the initialized model weights $\v w_j^{(0)}$ in CFL, analogous to the initial auxiliary variable $\v z_{j|i}^{(0)}$ in DFL. However, such secure transmission would not reduce the privacy loss in  \cref{eq.miCFLp}. 
The main reason is that at convergence all local models will be identical, i.e.,
$\v w_{j}^{(t_{\max})}=\v w_{k}^{(t_{\max})}$ for $(j,k)\in \mathcal{E}$. Thus, as long as there is one corrupt node, the passive adversary has knowledge of all $\v w_{j}^{(t_{\max})}$s. By inspecting \cref{eq.w_ave} we can see that the difference $\v w_{j}^{(t+1)}-\v w_{j}^{(t)}$ at every iteration is known, and thus $\v w_j^{(0)}$ for all $j\in{\cal V}$.

\end{remark}

\subsection{Privacy loss of DFL}
By inspection of Algorithm \ref{alg:pdmm}, the eavesdropping adversary can intercept all messages transmitted along non-secure channels, thus having access to:
\begin{align}\label{eq.knowDFLe}
    \{\Delta \v z_{j|k}^{(t+1)}\}_{(j,k)\in \mathcal{E},t\in \mathcal{T}}.
\end{align}
For any edge in the network, the transmitted information will be known by the passive adversary as long as one end node is corrupt. Accordingly, we define
 $\mathcal{E}_h=\{(j,k)\in \mathcal{V}_h\times \mathcal{V}_h\}$, $\mathcal{E}_c=\mathcal{E}\setminus \mathcal{E}_h$ as the set of honest and corrupt edges, respectively.  
Given that the passive adversary can collect all information obtained by the corrupt nodes, by inspecting Algorithm \ref{alg:pdmm} it thus has the knowledge of $\{\v x_j\}_{j\in \mathcal{V}_c}\cup\{ \v z_{j \mid k}^{(0)}, \Delta \v z_{j \mid k}^{(t+1)}\}_{(j, k) \in \mathcal{E}_c,t\in \mathcal{T}}$.  
Combining this with the eavesdropping knowledge in \cref{eq.knowDFLe} we conclude that the adversary has the following knowledge:
\begin{align*}
    \{\v x_j\}_{j\in \mathcal{V}_c}\cup\{ \v z_{j \mid k}^{(0)}\}_{(j, k) \in \mathcal{E}_c}\cup\{\Delta \v z_{j \mid k}^{(t+1)}\}_{(j, k) \in \mathcal{E},t\in \mathcal{T}}.
\end{align*}
The information loss of an honest node $i \in \mathcal{V}_h$'s private data is thus given by
\begin{align} \label{eq.miDFLp}
& I (X_i;\mathcal{O}_{\rm {DFL}})\nonumber\\
&=I( X_i; \{X_j\}_{j\in \mathcal{V}_c},\{ Z_{j \mid k}^{(0)}\}_{(j, k) \in \mathcal{E}_c},\{\Delta Z_{j \mid k}^{(t+1)}\}_{(j, k) \in \mathcal{E},t \in \mathcal{T}}),
\end{align}

We first give some initial results on the information that can be deduced by the adversary.
\begin{proposition} \label{prop.inter}
Let $\mathcal{G}_h=({\mathcal{V}}_h,\mathcal{E}_h)$ be the subgraph of $\mathcal{G}$ after eliminating all corrupt nodes. 
Let $\mathcal{G}_{ h,1},\ldots,\mathcal{G}_{ h,k_h}$ denote the components of $\mathcal{G}_h$ and let ${\mathcal{V}}_{ h,k}$ be the vertex set of $\mathcal{G}_{ h,k}$.  Without loss of generality, assume the honest nodes $i$ belong to the first honest component, i.e., $i\in \mathcal{V}_{h,1}$. 
The adversary has the following knowledge about node $i\in{\cal V}_{h,1}$:
\begin{itemize}
\item[i)] Noisy local gradients: 
\begin{align} \label{eq.graNoisyz}
    \forall t\in \mathcal{T}: \nabla f_i(\v  w_i^{(t)})+\tsum_{k \in \mathcal{N}_{i,h}} \v B_{i|k} \v z_{i|k}^{(0)}
\end{align}
\item[ii)] Difference of local gradients: 
\begin{align}\label{eq.graDif_i}
 \forall t\in \mathcal{T}: \nabla f_i(\v  w_i^{(t+1)})- \nabla f_i(\v w_i^{(t)}),
\end{align}
\item[iii)] The aggregated sum of local gradients in honest component $\mathcal{G}_{h,1}$: 
\begin{align}\label{eq.graVh1}
 \forall t\in \mathcal{T}: \sum_{j\in \mathcal{V}_{h,1} }\nabla f_j(\v w_j^{(t)}).
\end{align}  
\end{itemize}
\end{proposition}
\begin{proof}
    See Appendix~\ref{pf.prop}.  
\end{proof}
We now proceed to present the main result of this paper. In particular, we will show that the information loss of DFL is dependent on the variance of initialized $\v z^{(0)}$, i.e., $\sigma^2_Z$. Notably,  in the special case where $\sigma^2_Z=0$,  all local gradients will be exposed similar to the case of centralized FL. 
In contrast, when $\sigma^2_Z$ approaches infinity,  the term \cref{eq.graNoisyz} contains no information about the private data or the local gradient, theoretically leading to $ I (X_i;\nabla f_i(  W_i^{(t)})+\tsum_{k \in \mathcal{N}_{i,h}} \v B_{i|k}  Z_{i|k}^{(0)})=0$. 
More specifically, we have the following privacy bounds.
\begin{theorem}[Privacy bounds of DFL] \label{thm_pas2}
We have
\begin{align}\label{eq.miDFLp_inter} 
I& (X_i; \mathcal{O}_{\rm {CFL}} ) \stackrel{(a)}{\geq} I (X_i;\mathcal{O}_{\rm {DFL}} )
   \nonumber\\
&= I \big( X_i; \{X_j\}_{j\in \mathcal{V}_c},\{ Z_{j \mid k}^{(0)}\}_{(j, k) \in \mathcal{E}_c},\{W_j^{(t)}\}_{j\in \mathcal{V}, t\in \mathcal{T}},\big.\nonumber\\
  &\quad\big.\{\nabla f_j(W_j^{(t)})\}_{j\in \mathcal{V}_c,t\in \mathcal{T}}, \{Z_{j|k}^{(0)}-Z_{k|j}^{(0)}\}_{(j,k)\in\mathcal{E}_h} ,\big.\nonumber\\
  &\quad\big.\{\nabla f_j(W_j^{(t)}) + \tsum_{k \in \mathcal{N}_{j,h}} \v B_{j|k} Z_{j|k}^{(0)}\}_{j\in \mathcal{V}_h,t\in \mathcal{T}}\big) \\
  &\stackrel{(b)}{\geq} I\big( X_i; \{X_j\}_{j\in \mathcal{V}_c},\{W_j^{(t)}\}_{j\in \mathcal{V}, t\in \mathcal{T}},\{\nabla f_j(W_j^{(t)})\}_{j\in \mathcal{V}_c,t\in \mathcal{T}}, \big.\nonumber\\
&\quad\big.\{\nabla f_j(W_j^{(t+1)})-\nabla f_j(W_j^{(t)})\}_{j\in \mathcal{V}_h,t\in \mathcal{T}}, \big.\nonumber\\
&\quad\big.\{\tsum_{j\in \mathcal{V}_{h,l}}\nabla f_j(W_j^{(t)})\}_{1\leq l \leq k_h, t\in \mathcal{T}}\big),\label{eq.miLower} 
\end{align}
where we have equality in $(a)$ if $\sigma^2_Z=0$, and equality in $(b)$ if $\sigma^2_Z\rightarrow \infty$.
\end{theorem}

\begin{proof}
 See Appendix~\ref{pf.thm_pas2}.  
\end{proof}
Hence, by inspecting the lower bound we can see that 
except for the knowledge of the corrupt nodes, i.e.,$\{\v x_j\}_{j\in \mathcal{V}_c},\{\v w_j^{(t)}\}_{j\in \mathcal{V}_c, t\in \mathcal{T}},\{f_j(\v w_j^{(t)})\}_{j\in \mathcal{V}_c,t\in \mathcal{T}}$, the revealed information  includes the model weights of all honest nodes $\{\v w_j^{(t)}\}_{j\in \mathcal{V}_h,t\in \mathcal{T}}$,  gradient differences of each honest node over successive iterations $\{\nabla f_j(\v  w_j^{(t+1)})- \nabla f_j(\v w_j^{(t)})\}_{j\in \mathcal{V}_h,t\in \mathcal{T}}$ (note that $\cup_{1\leq l \leq k_h} \mathcal{V}_{h,l}=\mathcal{V}_h$), and the sum of local gradients of the honest nodes $\{\tsum_{j\in \mathcal{V}_{h,l}}\nabla f_j(\v w_j^{(t)})\}_{1\leq l \leq k_h,t\in \mathcal{T}}$ in each component.

Regarding the feasibility of the lower bound, we have the following remark.
\begin{remark} \label{rm.graNoisy}
It might seem impractical that the lower bound in Theorem\ref{thm_pas2} requires that the variance $\sigma^2_Z$ approaches infinity.
For practical applications, however,  like DNNs,  a relatively small variance is already sufficient to make the leaked information in the noisy gradients negligible compared to gradient differences \cref{eq.graDif_i} and the gradient sum \cref{eq.graVh1} (we will verify this claim in Section \ref{subsec:opt}). 
\end{remark}

\subsection{Privacy gap between CFL and DFL}
Theorem~\ref{thm_pas2} shows that the privacy loss in DFL is either less than or equal to that in CFL. This naturally leads to key questions: under what circumstances does this equality or inequality hold?
To analyze the privacy gap we have the following result, showing that the privacy gap between DFL and CFL is dependent on how much more information about the private data $X_i$ can the local gradients reveal given the knowledge of DFL, e.g., gradient differences and the gradient sum. 
\begin{corollary}\label{cor.mi}
\textbf{Privacy gap between CFL and DFL}
If the lower bound \cref{eq.miLower} is achieved, for an honest node $i \in \mathcal{V}_h$'s private data, the privacy gap between CFL and DFL is  given by 
\begin{align}\label{eq.miGap}
&I(X_i;\mathcal{O}_{\rm {CFL}})-I(X_i;\mathcal{O}_{\rm {DFL}})\nonumber\\
&=I\big( X_i;\{\nabla f_j(W_j^{(t)})\}_{j\in \mathcal{V}_h,t\in \mathcal{T}}|\{X_j\}_{j\in \mathcal{V}_c},\{W_j^{(t)}\}_{j\in \mathcal{V}, t\in \mathcal{T}},\big.\nonumber\\
&\quad\big. \{\nabla f_j(W_j^{(t+1)})-\nabla f_j(W_j^{(t)})\}_{j\in \mathcal{V}_h,t\in \mathcal{T}}, \big.\nonumber\\
&\quad\big.\{\tsum_{j\in \mathcal{V}_{h,l}}\nabla f_j(W_j^{(t)})\}_{1\leq l \leq k_h, t\in \mathcal{T}}\big)  
\end{align}
\end{corollary}
\begin{proof}
    See Appendix~\ref{pf.cor}.  
\end{proof}

\begin{remark} \label{rm.dpcase}
The privacy gap, as defined in \cref{eq.miGap}, depends on the number of corrupt nodes and narrows notably in the extreme case where only one honest node remains. By inspecting \cref{eq.miGap}, we can see that the privacy gap is intrinsically linked to the sum of gradients from honest nodes within each honest component, specifically, $\{\tsum_{j\in \mathcal{V}_{h,l}}\nabla f_j(W_j^{(t)})\}_{1\leq l \leq k_h, t\in \mathcal{T}}$. As the number of honest nodes diminishes, the specificity of information conveyed by individual node gradients increases, consequently reducing the privacy gap.  In the most extreme scenario, where only one node is honest, i.e., $\mathcal{V}_h=\{i\}$, the privacy gap reduces to zero since \cref{eq.miGap}$=I\big( X_i;\{\nabla f_i(W_i^{(t)})\}_{t\in \mathcal{T}}|\{\nabla f_i(W_i^{(t)})\}_{t\in \mathcal{T}}\big)=0$. 
\end{remark}

In the following sections, we delve into two distinct cases to further investigate the privacy gap. First, we examine a straightforward logistic regression example, where it's possible to analytically calculate the privacy loss. In this instance, we find no discernible privacy gap. Next, we shift our focus to a more conventional application within FL: DNNs.  Through extensive empirical analysis, we observe a notable privacy gap between CFL and DFL. This gap highlights DFL's reduced susceptibility to privacy attacks when compared to CFL.

\begin{figure}[ht]
    \centering
    \includegraphics[width=0.5\textwidth]{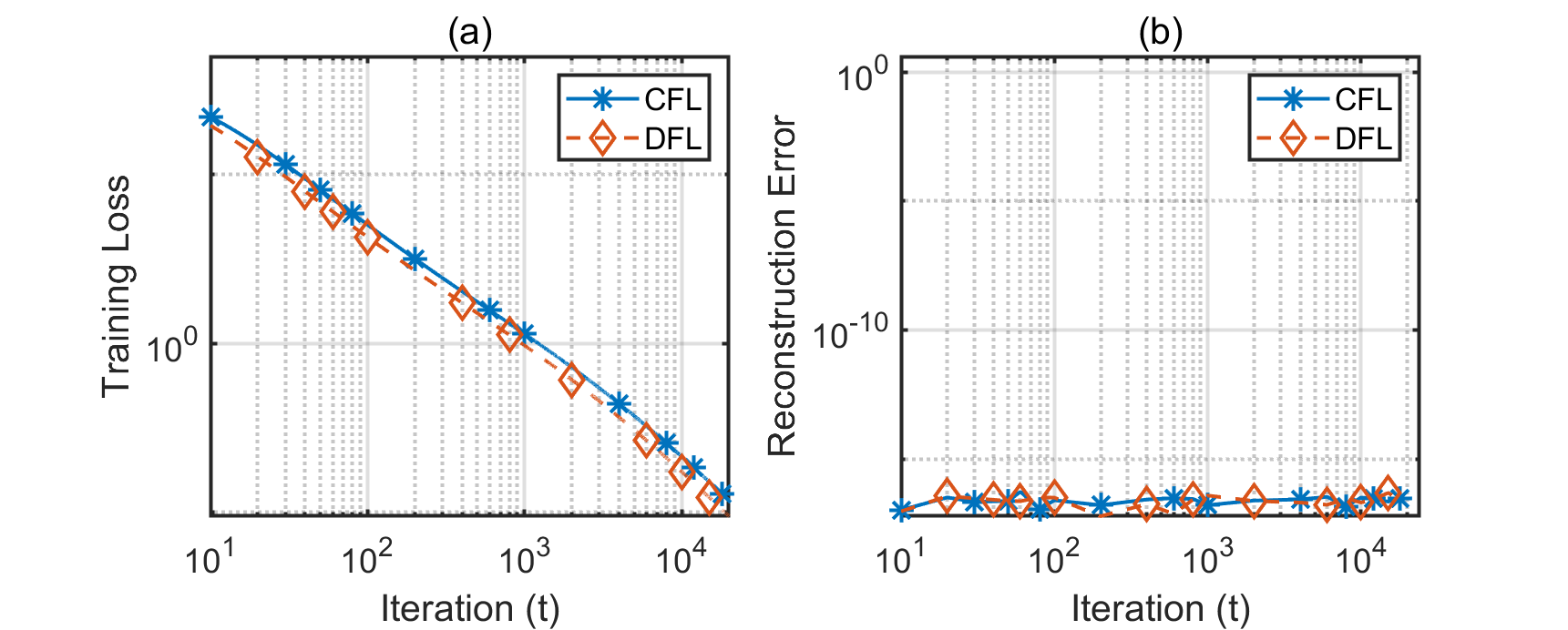}
    \vskip -8pt
    \caption{Privacy comparisons of centralized and decentralized logistic regression. (a) Training loss and (b) Reconstruction error of input data as a function of iteration number $(t)$ using CFL (blue color) and DFL  (red color). } 
    \label{fig:loss_error}
\end{figure}

\section{Logistic regression}
\label{sec:experimentI}
Logistic regression is widely adopted in various applications and serves as a fundamental building block for complex applications.  We will first give analytical derivations and then demonstrate numerical results. 

\subsection{Theoretical analysis}
Consider a logistic model with model parameters $\v w_{i}\in\mathbb{R}^v$ (weights) and $b_i\in\mathbb{R}$ (bias)
where each node has a local dataset $\{(\v x_{ik}, \ell_{ik}) : k=1,\ldots,n_i\}$,  where $\v x_{ik} \in \mathbb{R}^v$ is an input sample, $\ell_{ik} \in \{0,1\}$ is the associated label. In addition, let $y_{ik} = \v w_{i}^T \v x_{ik}+b_i$ denote the output of the model given the input $\v x_{ik}$.
Note that the bias term can be included in the weight vector. Here we explicitly separate the bias from the true network weights as it will lead to more insight into how to reconstruct the input data from the observed gradient information. Correspondingly, $\v z_{i|j}$ is also separated into $\v z_{w,i|j}$ and $z_{b,i|j}$. With this, the loss (log-likelihood) function  has the form 
\begin{align}\label{eq:costlr}
    f_i(\v w_i,b_i) &= - \tsum_{k=1}^{n_{i}}
\big( \ell_{i k} \log \frac{1}{1+e^{-y_{ik}}} \nonumber \\
&\quad+(1-\ell_{i k}) \log \frac{e^{-y_{ik}}}{1+e^{-y_{ik}}}\big).
\end{align}
With this, \cref{eq.graDif_i} becomes
\begin{align}
    & \frac{\partial f_i}{\partial \v w_{i}}^{\!\!\!(t+1)}-\frac{\partial f_i}{\partial \v w_{i}}^{\!\!\!(t)} = \tsum_{k=1}^{n_{i}}\big( \frac{1}{1+e^{-y^{(t+1)}_{ik}}} - \frac{1}{1+e^{-y^{(t)}_{ik}}}\big) \v x_{ik} \nonumber \\
     &= -\tsum_{j \in \mathcal{N}_i} \v B_{i|j} \Delta \v z_{w,i|j}^{(t)}+\rho d_i\big( \v w_i^{(t)} - \v w_i^{(t+1)} \big), \label{equ:bound1_lr}
\end{align}
and
\begin{align}
     &\frac{\partial f_i}{\partial b_{i}}^{\!\!\!(t+1)}-\frac{\partial f_i}{\partial b_{i}}^{\!\!\!(t)} = \tsum_{k=1}^{n_{i}}\big( \frac{1}{1+e^{-y^{(t+1)}_{ik}}} - \frac{1}{1+e^{-y^{(t)}_{ik}}}\big) \nonumber \\
     &= -\tsum_{j \in \mathcal{N}_i} \v B_{i|j} \Delta  z_{b,i|j}^{(t)}+\rho d_i\big( b_i^{(t)} - b_i^{(t+1)} \big)\label{equ:bound2_lr},
\end{align} 
where all terms in the RHS are known by the adversary as the differences of the local model $\v w_i^{(t+1)} - \v w_i^{(t)}$ can be determined from $\Delta {\v z}_{j|i}^{(t+1)}-\Delta {\v z}_{i|j}^{(t)}$ by considering two successive $\v z$ updates of \cref{eq.zupNQ}.  As a special case where $n_i=1$, \cref{equ:bound1_lr} is just a scaled version of 
$\v x_{ik}$ where the scaling is given by \cref{equ:bound2_lr}. Hence, with gradient difference, we can analytically compute $\v x_{ik}$ as \begin{align}
   \v x_{ik} = \frac{-\tsum_{j \in \mathcal{N}_i} \v B_{i|j} \Delta \v z_{w,i|j}^{(t)}+\rho d_i\big( \v w_i^{(t)} - \v w_i^{(t+1)} \big)}{-\tsum_{j \in \mathcal{N}_i} \v B_{i|j} \Delta  z_{b,i|j}^{(t)}+\rho d_i\big( b_i^{(t)} - b_i^{(t+1)} \big)}\nonumber.
\end{align}
Hence, in this case, one gradient difference at an arbitrary iteration is sufficient to reveal all information about the private data, i.e.,    $I(X_i;\nabla f_j(W_j^{(t+1)})-\nabla f_j(W_j^{(t)}))=I(X_i;X_i)$ which is maximum. Thus, there is no privacy gap between CFL and DFL, i.e., $\cref{eq.miGap}=0$. For the case of $n_i>1$, the input $\v x_{ik}$ can also be reconstructed by searching for solutions that fit for given observations, i.e., \cref{equ:bound1_lr} and \cref{equ:bound2_lr} across iterations.

\subsection{Convergence behavior}
To validate the theory presented above, we consider a toy example of a random geometric graph of $n=60$ nodes randomly distributed in the unit cube having a communication radius $r = \sqrt{2\log n/n}$ to
ensure connectivity with high probability~\cite{dall2002random}. Detailed settings can be found in the Appendix \ref{app.conv}. 
In \autoref{fig:loss_error}(a) we demonstrate the convergence performances of both centralized and decentralized protocols, i.e.,  the loss \cref{eq:costlr}, averaged over all nodes, as a function of the iteration $t$. We can see that both methods have a similar convergence rate. 

\subsection{Privacy gap between CFL and DFL}
To evaluate the performance of input reconstruction, we define the reconstruction error as the average Euclidean distance between the reconstructed samples, denoted as $\hat{\v x}_{ik}$, and the original data samples $\v x_{ik}$ given by $\frac{1}{n} \tsum_{i=1}^n \|\v {\hat{x}}_{ik}-\v x_{ik}\|_2$. The corresponding errors are plotted as a function of iteration number in  \autoref{fig:loss_error}(b). We can see that, the reconstruction error of DFL, using gradient differences, has the same level of error as the CFL case across all iterations. 
Hence, we conclude that for the logistic regression example, there is no privacy gap between CFL and DFL, i.e.,  $\cref{eq.miGap}=0$,  aligning with our theoretical result.

\section{Deep neural networks I: Convergence and gradient inversion attacks}\label{sec:experimentII}
For DNNs, it is challenging to give an analytical analysis like the above logistic regression example.  We resort to empirical assessments of privacy leakage through privacy attacks.  As we shall see, the empirical evaluation results are consistent with our theoretical claims. 

\subsection{Convergence behavior}
To test the performance of the introduced DFL protocol, we make comparison between the training process of CFL and DFL. The detailed settings and results are shown in Appendix \ref{app.conv}. Notably, the performance of the decentralized protocol closely aligns with that of the centralized approach. This aligns with previous research findings~\cite{niwa2020edge}, which suggests that decentralized protocols perform comparably to centralized ones, particularly in scenarios with independently and identically distributed data. The subsequent section will focus on evaluating their privacy via gradient inversion attack and membership inference attack, highlighting the privacy advantages inherent in DFL.

\subsection{Gradient inversion attack in DFL} \label{subsec:dnnPriv}
Recall that in Proposition \ref{prop.inter} we identified three types of information directly related to the honest node's local gradient: noisy gradients \cref{eq.graNoisyz}, gradient differences \cref{eq.graDif_i}, and the gradient sum \cref{eq.graVh1}. 
Similar to the traditional gradient inversion attack described in \cref{eq.traInv}, inputs can be inverted from them as well. As an example with the gradient sum \cref{eq.graVh1} the corresponding optimization problem can be formulated as
\begin{align}\label{eq.propInv2}
&\{(\v x_i^{\prime *}, \v \ell_i^{\prime *})\}_{i\in \mathcal{V}_h} =\underset{\v x_i^{\prime}, \v \ell_i^{\prime}}{\arg \min }\big\|\tsum_{i\in \mathcal{V}_{h,1}}\nabla f_i(\v w_i^{(t)},(\v x'_i,\v \ell'_i))\nonumber \\
&\quad- \tsum_{i\in \mathcal{V}_{h,1}}\nabla f_i(\v w_i^{(t)},(\v x_i,\v \ell_i))\big\|^2,
\end{align}
Although the gradient inversion attack in the decentralized case looks similar to the traditional centralized case, there are some important differences, in particular with respect to label recovery and fidelity of reconstructed inputs, discussed in the following remarks.
 \begin{figure}[t]
    \centering
    \includegraphics[width=0.4\textwidth]{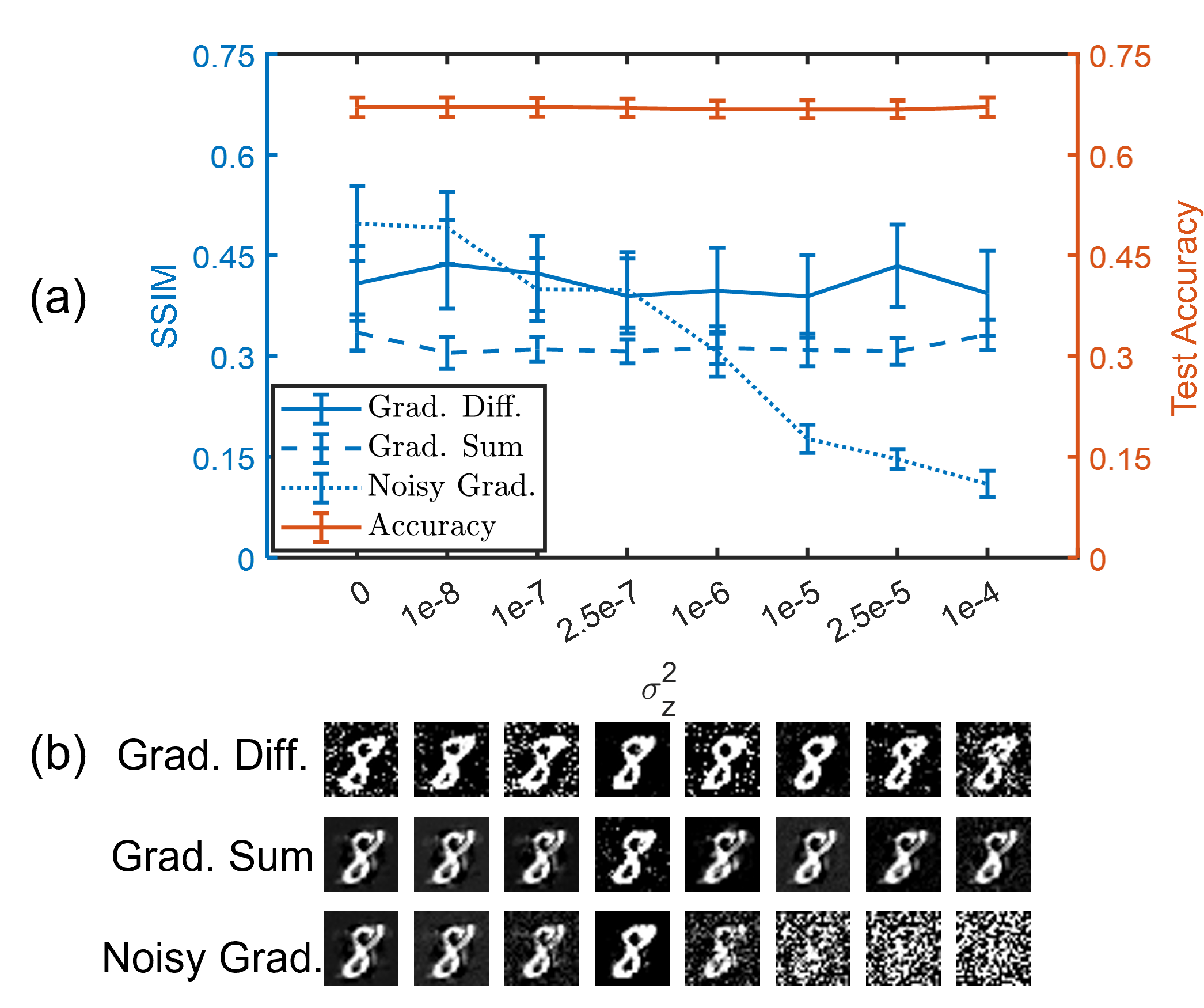}
    \vskip -6pt
    \caption{ (a) Averaged SSIM of reconstructed inputs by inverting noisy gradients, gradient differences, and the gradient sum (blue lines) and test accuracy (red line) for different variances of initialized auxiliary variable $\v z^{(0)}$:  $\sigma^2_Z=0,10^{-8},10^{-7},2.5\times10^{-7},10^{-6},10^{-5},2.5\times10^{-5}$ and $10^{-4}$. (b) Sample examples of reconstructed inputs for each case.}
    \label{fig:z_std}
    \vskip -6pt
\end{figure}

\begin{remark}\label{rm.graDif}
 \textbf{Analytical label recovery is no longer applicable given gradient differences \cref{eq.graDif_i}.}
With gradient differences, \cref{eq.label} explained in Section \ref{subsec.labelRecov} becomes 
\begin{align}
&\nabla f(\v w_{i,L,c}^{(t)})-\nabla f(\v w_{i,L,c}^{(t+1)}) = g_c^{(t)} \v a_{L-1}^{(t)}-g_c^{(t+1)} \v a_{L-1}^{(t+1)} \nonumber\\
&= \big( \frac{e^{y^{(t)}_c}}{\tsum_j e^{y^{(t)}_j}} -\delta_{c,\ell_i}\big)\v a_{L-1}^{(t)} -  \big( \frac{e^{y^{(t+1)}_c}}{\tsum_j e^{y^{(t+1)}_j}} - \delta_{c,\ell_i}\big) \v a_{L-1}^{(t+1)}. \nonumber
\end{align}
Hence, if $c\neq \ell_i$, then both $g_c^{(t)}<0$ and $g_c^{(t+1)}<0$ so that we cannot use the sign information of $g_c$ to recover the correct label. Therefore, the adversary needs to consider all labels to find out the best fit, which inevitably increases the computation overhead and degrades the fidelity of the reconstructed inputs (see Section \ref{ssubsec:graDif} for numerical validations). 
\end{remark}

\begin{remark}\label{rm.graVh1}
\textbf{Bigger component size $|\mathcal{V}_{h,1}|$ in the   gradient sum \cref{eq.graVh1} will make it more challenging to invert inputs.} This is due to the fact that \cref{eq.graVh1} is related to all data samples of all honest nodes in $\mathcal{V}_{h,1}$, thus inverting input from \cref{eq.propInv2}  is analogous to increase the batchsize or the total number of data samples in \cref{eq.traInv} of the centralized case. It has been empirically shown in many works~\cite{zhu2019deep,geiping2020inverting,yin2021see} that increasing batchsize will degrade the fidelity of reconstructed inputs severely. We will validate this result in  Section \ref{ssubsec:graVh1}.
\end{remark}

\subsection{Optimum attack strategy}\label{subsec:opt}
To test the performance of gradient inversion attack, 
we consider a random geometric graph with $n=50$ nodes. Each node randomly selects $n_i$ data samples from the corresponding dataset and uses a two-layer multilayer perceptron (MLP) to train the local model. For simplicity, in the following experiments, we set $n_i=2$, unless otherwise specified. We use the approach proposed in DLG~\cite{zhu2019deep} to conduct the gradient inversion attack  (see Appendix \ref{app.res} for results of using the cosine similarity-based approach proposed in \cite{geiping2020inverting}).
\begin{figure}[t]
    \centering
    \includegraphics[width=0.3\textwidth]{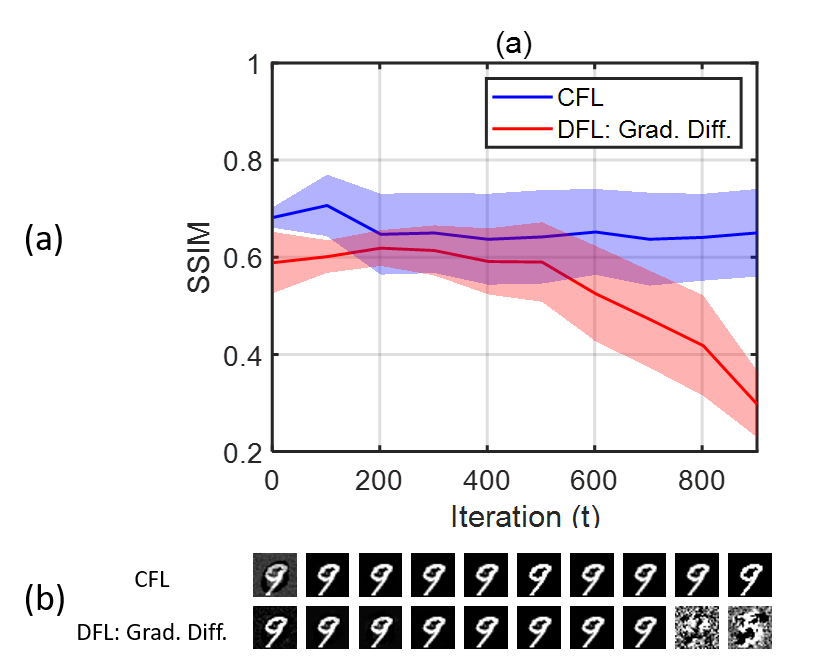}
    \vskip -6pt
    \caption{Performance of reconstructed inputs via inverting gradients (CFL) and gradient differences (DFL)  in terms of iterations $t$: (a) Averaged SSIM (solid lines) of all reconstructed inputs along with the corresponding standard derivation (shadows), (b) sample examples of reconstructed inputs at iteration number $t=1, 100,\ldots, 900$.}
    \label{fig:iter_ssim_diff}
\end{figure}

\begin{figure*}[t]
    \centering
    \includegraphics[width=0.8\textwidth]{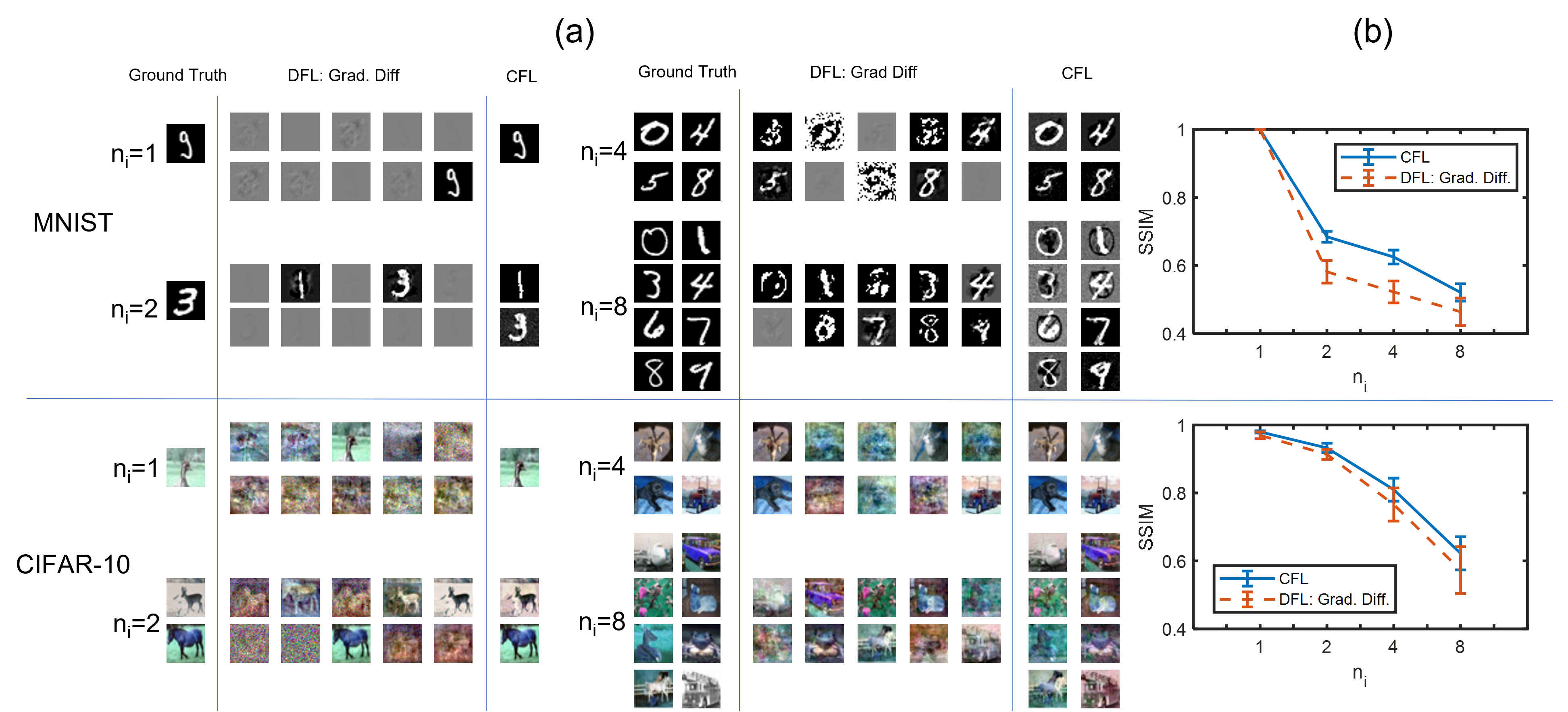}
    \vskip -8pt
    \caption{Performance comparisons of CFL and DFL via inverting inputs from gradient differences: (a) Samples images of ground truth and reconstructed inputs, (b) SSIM comparisons of all reconstructed inputs for different batch size $n_i=1,2,4,8$  using two datasets MNIST (top) and CIFAR-10 (bottom), respectively.}
    \label{fig:ni}
    \vskip -10pt
\end{figure*}

\begin{figure*}[t]
\centering
\includegraphics[width=.8\textwidth]{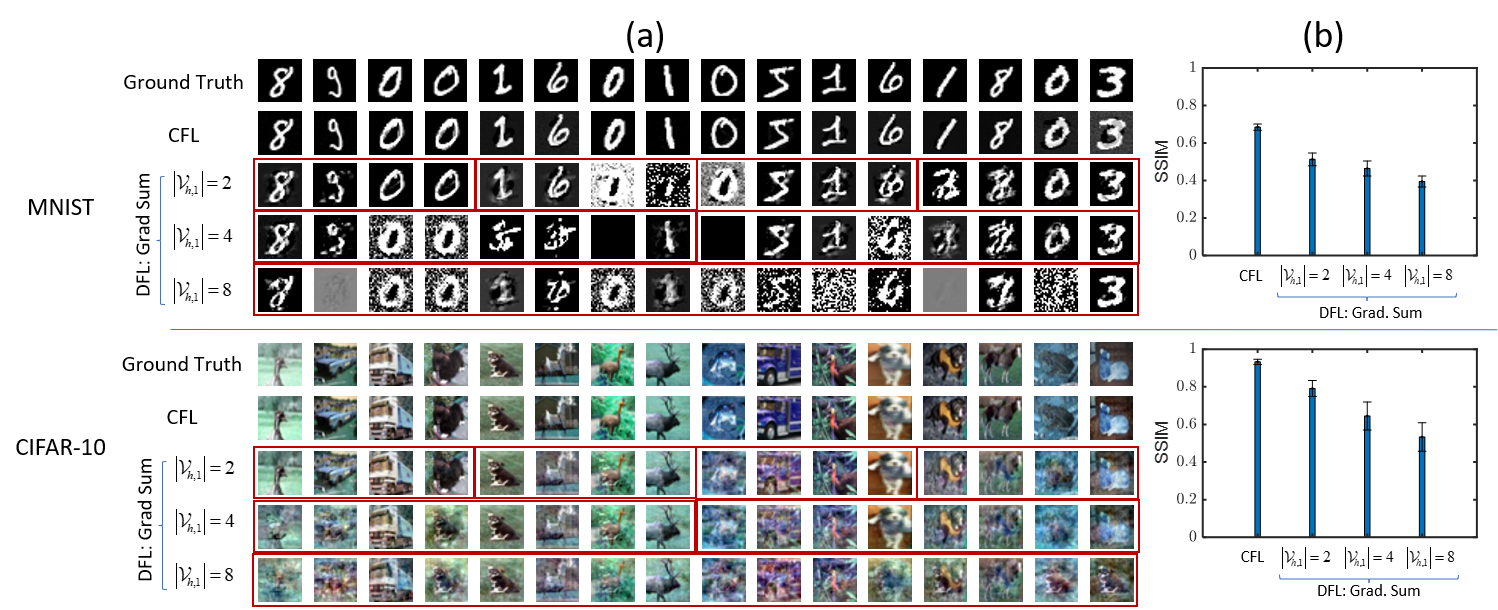}
\vskip -8pt
 \centering
 \mbox{}\vspace*{-.8\baselineskip}
\caption{Performance comparisons of CFL and DFL via inverting inputs from the gradient sum: (a) Sample images of ground truth and reconstructed inputs, (b) SSIM of all reconstructed samples for three different sizes of honest component $|\mathcal{V}_{h,1}|=2,~4,~8$ using two datasets MNIST (top) and CIFAR-10 (bottom), respectively. Wherein the red box indicates that the corresponding samples are from the same component.}
\label{fig:mi}
\vskip -6pt
\end{figure*}

\vspace{1mm}
\noindent\textbf{Noisy gradients vs. gradient differences vs. the gradient sum} 
As highlighted in Remark \ref{rm.graNoisy}, the noisy gradient term vanishes in the lower bound \cref{eq.miLower} if the variance $\sigma^2_Z\rightarrow \infty$. This raises an important question: in practice how large should $\sigma^2_Z$ be to ensure that the effectiveness of reconstructing inputs from noisy gradients is inferior to that of other variables, namely gradient differences and the gradient sum. To investigate this, \autoref{fig:z_std}(a) presents a comparative analysis of input reconstruction performances using noisy gradients, gradient differences, and the gradient sum (considering two honest nodes in the component, i.e., $|\mathcal{V}_{h,1}|=2$), along with the test accuracies of learn models for different choices using the MNIST dataset.   The results indicate that, without loss of test performance, the effectiveness of inverting inputs from noisy gradient degrades very fast as $\sigma^2_Z$ increases, performing worse than inverting inputs using gradient differences and the gradient sum even at a very low variance of $\sigma^2_Z=10^{-5}$. This is further illustrated in plot (b) where the ground truth digit $8$ can hardly be recognized when inverting noisy gradient for variance $\sigma^2_Z\geq 10^{-5}$. 
This suggests that in practical scenarios,  a small variance $\sigma^2_Z$ is sufficient to ensure the lower bound is attached. As a consequence,  in what follows we will evaluate the reconstruction performances via inverting gradient differences and the gradient sum for the case of DFL.

\vspace{1mm}
\noindent\textbf{Optimum attack iteration}
Given that both CFL and DFL protocols are iterative processes consisting of numerous iterations, in principle the gradient inversion attack can deploy all iterations' information for inverting the input samples. To identify the most effective attack strategy for such an attack, we first explore which iteration yields the most effective results for the gradient inversion attack. In \autoref{fig:iter_ssim_diff}(a) we demonstrate the attack performance, quantified by the averaged SSIM  of all reconstructed samples (illustrated by solid lines) along with their standard derivation (shown as shadows)
of CFL and DFL using gradient differences as a function of iteration number $t$ using the MNIST dataset. We can see that the SSIM of the centralized case consistently surpasses those in the decentralized case throughout all iterations. Unlike the centralized case where the reconstruction performance remains relatively stable across iterations, the reconstruction performances of the decentralized case degrade significantly in later iterations. This trend aligns with the expectation that gradient differences will approach zero at convergence thereby containing very little information. 

To visualize this phenomenon, \autoref{fig:iter_ssim_diff}(b) 
displays illustrative examples of both cases at every 100 iterations. The SSIM results are in agreement with these visual examples; for instance, in the later iterations of the decentralized case,  the ground truth digit $9$ can hardly be recognized. Hence, the effectiveness of gradient inversion attacks is more pronounced in the initial iterations. 
Note that for both CFL and DFL we observe a slightly better performance at the early stage. This might be due to the fact that at early iterations the local models contain more information about the local datasets while at later iterations the models are more fitted to the global dataset. For this reason, for the forthcoming Figure \ref{fig:ni} and \ref{fig:mi}, we will focus on exploiting gradients from early iterations to execute the gradient inversion attacks.

\section{Deep neural networks II: privacy gap}\label{sec:DNNII}
We now evaluate the privacy gap between CFL and DFL using different settings and privacy attacks.

\subsection{Evaluating the privacy gap between CFL and DFL using gradient inversion attack}
\vspace{1mm}
\noindent\textbf{Inverting gradient differences \cref{eq.graDif_i}} \label{ssubsec:graDif}
In Remark \ref{rm.graDif} we showed that using gradient differences labels cannot be analytically computed, unlike in CFL.  To assess the impacts of this on the performance of input reconstructions, \autoref{fig:ni}(a) showcases examples of reconstructed inputs for both the MNIST and CIFAR-10 dataset for four different batch-sizes, i.e., $n_i=1,2,4,8$. Notable disparities exist in the reconstructed samples of CFL and DFL cases.  More specifically, compared to the centralized case, reconstructing inputs via gradient differences is less efficient and yields lower quality.  The inefficiency is due to the fact that the adversary needs to iterate through all possible labels to identify the best fit. Furthermore, the quality is compromised due to the inherently reduced information in gradient differences compared to the full gradients. 
To evaluate the quality of reconstructed inputs, we chose the best-fit samples among iterated inputs for each batch and computed their averaged SSIM. The comparison results are presented in  \autoref{fig:ni}(b), illustrating the quality of reconstructed inputs of CFL are consistently better than that in DFL for all batchsizes $n_i$. This comparison underlines the inherent challenges in reconstructing high-quality inputs from gradient differences in DFL.

\vspace{1mm}
\noindent\textbf{Inverting the gradient sum \cref{eq.graVh1}} \label{ssubsec:graVh1}
When inverting inputs from the gradient sum, it is intuitive that the accuracy of reconstructed inputs tends to diminish as the size of the honest component $|\mathcal{V}_{h,1}|$ increases, analogous to the situation of increasing batchsize. 
This relationship is clearly illustrated in \autoref{fig:mi} for the MNIST dataset, showing a direct correlation between the size of the honest component and the precision of the reconstructed inputs. 
It is important to note that while obtaining label information becomes more challenging in DFL (given the gradient sum),  in this comparison we assume that the label information is known a prior for both CFL and DFL. 
However, even with this assumption, the reconstruction performances of DFL are still notably inferior to that of CFL.

\vspace{1mm}
\noindent\textbf{Combining both gradient differences and the gradient sum} 
Since both gradient differences and the gradient sum are available in DFL, one natural question to ask is if combining them together will improve the attack performances. 
We demonstrate the comparison results in \autoref{fig:comb} in Appendix \ref{app.res},  showing the SSIM of reconstructed inputs of CFL and DFL across varying iteration numbers for three different sizes. The results suggest that combining gradient differences and gradient sum does not amplify the reconstruction performances. Still, the attack performances of CFL are in general better than DFL, consistently across all iterations. 

\vspace{1mm}
\subsection{Evaluating privacy gap with additional gradient inversion attacks, model architectures and datasets} 
In addition to the MNIST and CIFAR-10 dataset and MLP network, we further evaluated the privacy gap using two additional datasets (CIFAR-100 and Tiny ImageNet), two additional model architectures (VGG-11 and AlexNet), and two additional gradient inversion attacks: the cosine similarity-based method proposed in \cite{geiping2020inverting} (see \autoref{fig:vgg}, \autoref{fig:graddiff_cos} and \autoref{fig:gradsum_cos}) and the generative model-based method proposed in \cite{jeon2021gradient} (see \autoref{fig:generative}).   
The results indicate that the performance of DFL is generally lower than that of CFL, which aligns with the findings using DLG \cite{zhu2019deep}, as shown in \autoref{fig:ni} and \autoref{fig:mi}. These results corroborate our theoretical analysis and suggest that the gap between CFL and DFL is consistent across different gradient inversion attacks, datasets, and model architectures.

\subsection{Evaluating privacy gap using  membership inference attacks} \label{subsec.mia}
Apart from gradient inversion attacks, it is also interesting to see if there is a privacy gap between CFL and DFL when evaluated using membership inference attacks (MIAs). The attack results and detailed experimental settings are demonstrated in Figure \ref{fig:roc_gradient}. Similar to the case of using gradient inversion attacks, we can see that CFL leaks more membership information compared to DFL. 

\begin{figure}[ht]
    \centering
\includegraphics[width=0.25\textwidth]{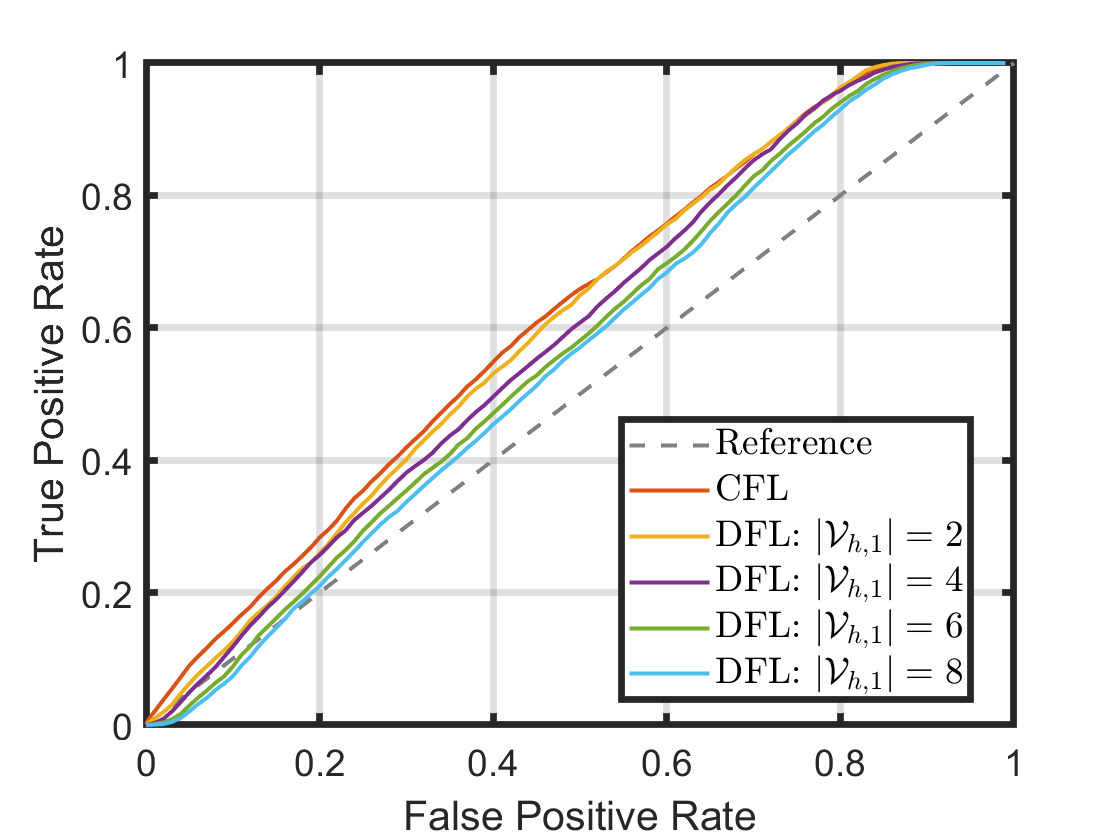}
    \vskip -6pt
    \caption{Membership inference attack results comparisons of CFL and DFL with gradient sum using the gradient-based approach \cite{li2022effective}}
    \label{fig:roc_gradient}
\end{figure}

\begin{figure}[t]
    \centering
    \includegraphics[width=0.3\textwidth]{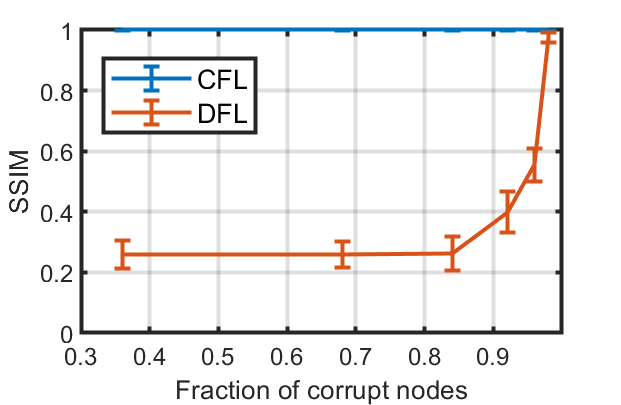}
    \vskip -6pt
    \caption{Performance comparisons of CFL and DFL in terms of the portion of corrupted nodes. } 
    \label{fig:susscess_rate}
\end{figure}

\subsection{Privacy gap reduces with more corrupt nodes}\label{subsec.mia2}
We now validate the result presented in Remark \ref{rm.dpcase}. As shown in Figure \ref{fig:susscess_rate} when using MNIST datasets, the privacy gap between CFL and DFL predictably narrows as the number of corrupt nodes increases when using gradient inversion attacks.  Notably, in the case where there is only a single honest node, the attack performance in DFL converges with that of CFL, effectively closing the privacy gap.  As for the case of using membership inference attacks, we observe a similar tendency (see Figure \ref{fig:mia}). Hence, these results are consistent with our theoretical findings in Remark \ref{rm.dpcase}, thereby confirming the validity of our analytical approach. 

\begin{figure}[ht]
    \centering
\includegraphics[width=0.3\textwidth]{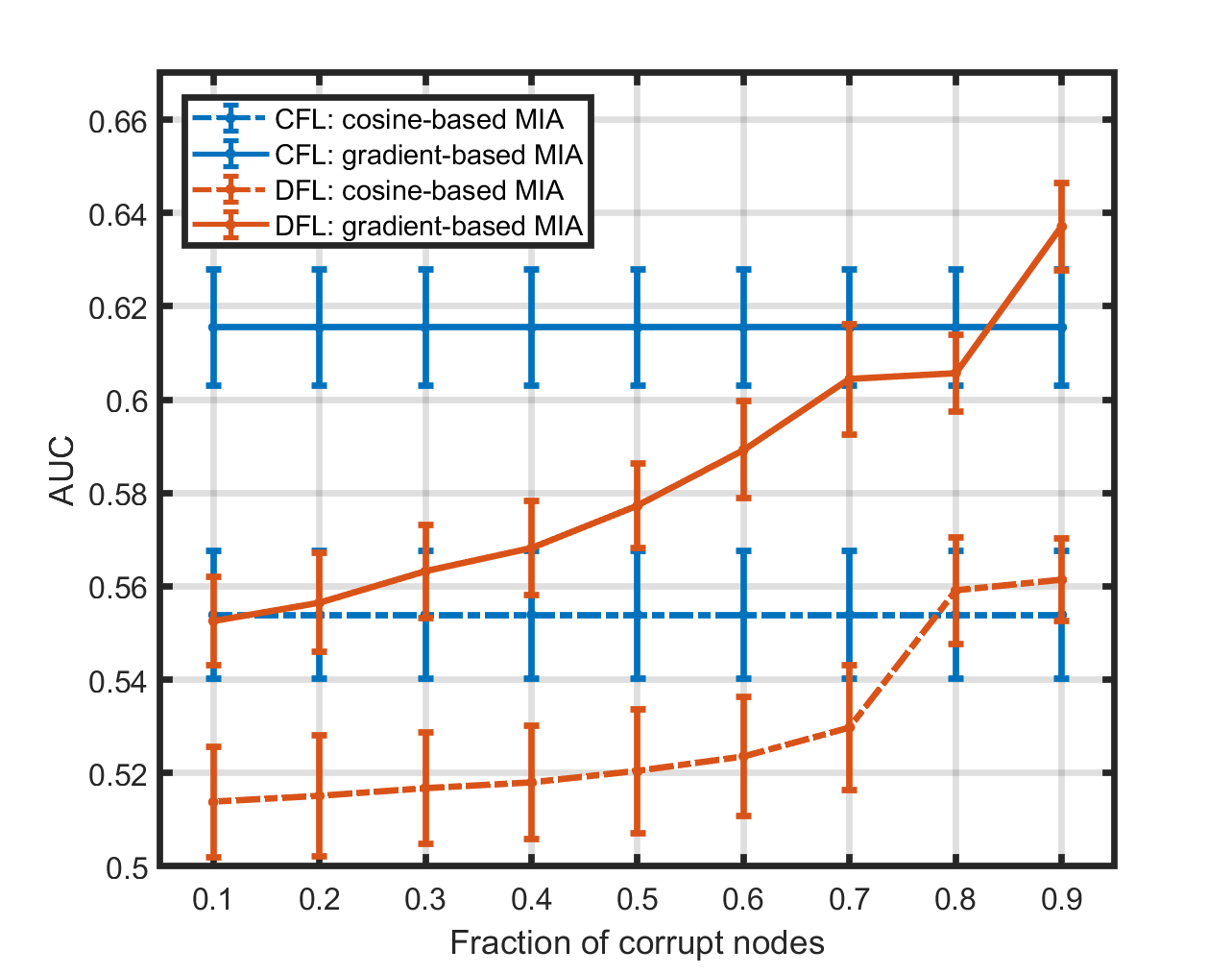}
\vskip -6pt
    \caption{Performance comparisons of CFL and DFL in terms of the portion of
corrupted nodes via membership inference attacks.}
    \label{fig:mia}
\end{figure}

Overall,  we conclude that compared to the CFL case,  the 
DFL protocol is less vulnerable to privacy attacks including gradient inversion attacks and membership inference attacks.

\subsection{Related work} \label{ssec.epfl}
The recent work~\cite{pasquini2022privacy} highlights that prior comparisons between the privacy implications of CFL and DFL either lack empirical evidence or fail to comprehensively explore privacy arguments. Consequently, ~\cite{pasquini2022privacy} stands as the only direct and relevant benchmark for our analysis, which contends that DFL offers no privacy advantages over CFL. However, our study provides a new perspective and challenges this conclusion. We hypothesize that these discrepancies may be attributed to the following factors:
\begin{enumerate}
    \item \textbf{Decentralization Techniques:} The decentralization techniques employed are significantly different. While ~\cite{pasquini2022privacy} utilizes average consensus-based decentralization methods (refer to Section \ref{sssec:avgDFL}), our study focuses on distributed optimization techniques (see Section \ref{sssec:optDFL}). The former approach separates local and global model updates similar to CFL, whereas the latter integrates these updates into a joint optimization process.
    \item \textbf{Threat Models:} ~\cite{pasquini2022privacy} examines both passive and active adversary models. In contrast, our study considers eavesdropping and passive adversary models, offering a different perspective on potential security risks.
    \item \textbf{Theoretical and Empirical Results:} While ~\cite{pasquini2022privacy} confines its findings to empirical data, primarily evaluating privacy risks through membership inference attacks, our research provides a comprehensive analysis includes both information-theoretical analysis and empirical validations. These validations encompass gradient inversion attacks as well as membership inference attacks.
\end{enumerate}
Given these significant methodological divergences, direct comparisons between our findings and those of ~\cite{pasquini2022privacy} are inherently challenging. This highlights the necessity for more detailed and extensive investigations. We advocate for continued research in this domain to unravel the complex dynamics that contribute to these divergent outcomes and to deepen our understanding of privacy mechanisms within FL frameworks.

\section{Conclusions}
\label{sec:conclusion}
In this paper, we showed that DFL through distributed optimization inherently provides privacy advantages compared to CFL, particularly in complex settings like neural networks.  We conducted a detailed analysis of the information flow within the network across various iterations, establishing both upper and lower bounds for the information loss. The bounds indicate that privacy leakage in DFL is consistently less than or equal to that in CFL. We further exemplified our results through two standard applications: logistic regression and training a DNN.  In the case of simple logistic regression, we observed that the privacy leakage in both CFL and DFL are identical. However, in more complex scenarios like training DNNs, the privacy loss as measured by the gradient inversion attack is markedly higher in CFL than in DFL. As expected, the privacy gap between CFL and DFL is more pronounced in the presence of numerous honest nodes.   Extensive experimental results substantiated our findings.

\bibliography{refer.bib}

\begin{thebibliography}{10}

\bibitem{mcmahan2017communication}
B.~McMahan, E.~Moore, D.~Ramage, S.~Hampson, and B.~A. y~Arcas, ``Communication-efficient learning of deep networks from decentralized data,'' in {\em Proc. Int. Conf. Artif. Intell. Statist.}, pp.~1273--1282, PMLR, 2017.

\bibitem{li2020federated}
{T. Li, A. K. Sahu, A. Talwalkar and V. Smith}, ``Federated learning: Challenges, methods, and future directions,'' {\em IEEE Signal Process. Mag.}, vol.~37, no.~3, pp.~50--60, 2020.

\bibitem{jin2016scale}
P.~H. Jin, Q.~Yuan, F.~Iandola, and K.~Keutzer, ``How to scale distributed deep learning?,'' {\em arXiv preprint arXiv:1611.04581}, 2016.

\bibitem{lian2017can}
{X. Lian, C. Zhang, H. Zhang, C. Hsieh, W. Zhang and J. Liu}, ``Can decentralized algorithms outperform centralized algorithms? a case study for decentralized parallel stochastic gradient descent,'' in {\em Proc. Adv. Neural Inf. Process. Syst.}, vol.~30, 2017.

\bibitem{tang2018d}
H.~Tang, X.~Lian, M.~Yan, C.~Zhang, and J.~Liu, ``D$^2$: Decentralized training over decentralized data,'' in {\em Proc. Int. Conf. Mach. Learn.}, pp.~4848--4856, PMLR, 2018.

\bibitem{hu2019decentralized}
C.~Hu, J.~Jiang, and Z.~Wang, ``Decentralized federated learning: A segmented gossip approach,'' {\em arXiv preprint arXiv:1908.07782}, 2019.

\bibitem{mota2013d}
J.~M. J.~M. Xavier, P.~M. Aguiar, and M.~P{\"u}schel, ``D-admm: A communication-efficient distributed algorithm for separable optimization,'' {\em IEEE Trans. Signal Process.}, vol.~61, no.~10, pp.~2718--2723, 2013.

\bibitem{li2019communication}
W.~Li, Y.~Liu, Z.~Tian, and Q.~Ling, ``Communication-censored linearized admm for decentralized consensus optimization,'' {\em IEEE Trans. Signal Inf. Process. Netw.}, vol.~6, pp.~18--34, 2019.

\bibitem{chen2021coded}
H.~Chen, Y.~Ye, M.~Xiao, M.~Skoglund, and H.~Poor, ``Coded stochastic admm for decentralized consensus optimization with edge computing,'' {\em IEEE Internet Things J.}, vol.~8, no.~7, pp.~5360--5373, 2021.

\bibitem{zhang2017distributed}
G.~Zhang and R.~Heusdens, ``Distributed optimization using the primal-dual method of multipliers,'' {\em IEEE Trans. Signal Inf. Process. Netw.}, vol.~4, no.~1, pp.~173--187, 2017.

\bibitem{sherson2018derivation}
{T. Sherson, R. Heusdens, W. B. Kleijn}, ``Derivation and analysis of the primal-dual method of multipliers based on monotone operator theory,'' {\em IEEE Trans. Signal Inf. Process. Netw.}, vol.~5, no.~2, pp.~334--347, 2018.

\bibitem{niwa2020edge}
K.~Niwa, N.~Harada, G.~Zhang, and W.~Kleijn, ``Edge-consensus learning: Deep learning on p2p networks with nonhomogeneous data,'' in {\em Proc. ACM SIGKDD Int. Conf. Knowl. Discovery Data Mining}, pp.~668--678, 2020.

\bibitem{zhu2019deep}
L.~Zhu, Z.~Liu, and S.~Han, ``Deep leakage from gradients,'' in {\em Proc. Adv. Neural Inf. Process. Syst.}, vol.~32, 2019.

\bibitem{zhao2020idlg}
B.~Zhao, K.~Mopuri, and H.~Bilen, ``i{DLG}: Improved deep leakage from gradients,'' {\em arXiv preprint arXiv:2001.02610}, 2020.

\bibitem{geiping2020inverting}
J.~G. H.~Bauermeister, H.~Dr{\"o}ge and M.~Moeller, ``Inverting gradients-how easy is it to break privacy in federated learning?,'' in {\em Proc. Adv. Neural Inf. Process. Syst.}, vol.~33, pp.~16937--16947, 2020.

\bibitem{yin2021see}
H.~Yin, A.~Mallya, A.~Vahdat, J.~Alvarez, J.~Kautz, and P.~Molchanov, ``See through gradients: Image batch recovery via gradinversion,'' in {\em Proc. IEEE Conf. Comput. Vis. Pattern Recognit.}, pp.~16337--16346, 2021.

\bibitem{boenisch2021curious}
F.~Boenisch, A.~Dziedzic, R.~Schuster, A.~Shamsabadi, I.~Shumailov, and N.~Papernot, ``When the curious abandon honesty: Federated learning is not private,'' {\em arXiv preprint arXiv:2112.02918}, 2021.

\bibitem{geng2023improved}
J.~Geng, Y.~Mou, Q.~Li, F.~Li, O.~Beyan, S.~Decker, and C.~Rong, ``Improved gradient inversion attacks and defenses in federated learning,'' {\em IEEE Trans. Big Data}, 2023.

\bibitem{wei2020framework}
W.~Wei, L.~Liu, M.~Loper, K.-H. Chow, M.~Gursoy, S.~Truex, and Y.~Wu, ``A framework for evaluating gradient leakage attacks in federated learning,'' {\em arXiv preprint arXiv:2004.10397}, 2020.

\bibitem{yang2022using}
H.~Yang, M.~Ge, K.~Xiang, and J.~Li, ``Using highly compressed gradients in federated learning for data reconstruction attacks,'' {\em IEEE Trans. Inf. Forensics Secur.}, vol.~18, pp.~818--830, 2022.

\bibitem{zhao2022deep}
Z.~Zhao, M.~Luo, and W.~Ding, ``Deep leakage from model in federated learning,'' {\em arXiv preprint arXiv:2206.04887}, 2022.

\bibitem{xu2022agic}
J.~Xu, C.~Hong, J.~Huang, L.~Y. Chen, and J.~Decouchant, ``Agic: Approximate gradient inversion attack on federated learning,'' in {\em Proc. 41st Int. Symp. Rel. Distrib. Syst.}, pp.~12--22, IEEE, 2022.

\bibitem{yuan2016convergence}
K.~Yuan, Q.~Ling, and W.~Yin, ``On the convergence of decentralized gradient descent,'' {\em SIAM J. Optim.}, vol.~26, no.~3, pp.~1835--1854, 2016.

\bibitem{cheng2019towards}
H.~Cheng, P.~Yu, H.~Hu, S.~Zawad, F.~Yan, S.~Li, H.~Li, and Y.~Chen, ``Towards decentralized deep learning with differential privacy,'' in {\em Proc. Int. Conf. Cloud Comput.}, pp.~130--145, Springer, 2019.

\bibitem{vogels2021relaysum}
T.~Vogels, L.~He, A.~Koloskova, S.~P. Karimireddy, T.~Lin, S.~U. Stich, and M.~Jaggi, ``Relaysum for decentralized deep learning on heterogeneous data,'' in {\em Proc. Adv. Neural Inf. Process. Syst.}, vol.~34, pp.~28004--28015, 2021.

\bibitem{pasquini2022privacy}
D.~Pasquini, M.~Raynal, and C.~Troncoso, ``On the (in) security of peer-to-peer decentralized machine learning,'' in {\em Proc. IEEE Symp. Secur. Privacy}, pp.~418--436, 2023.

\bibitem{dimakis2010gossip}
{A. G. Dimakis, S. Kar, J. M. Moura, M. G. Rabbat, and A. Scaglione}, ``Gossip algorithms for distributed signal processing,'' {\em Proc. IEEE}, vol.~98, no.~11, pp.~1847--1864, 2010.

\bibitem{olshevsky2009convergence}
{A. Olshevsky and J. Tsitsiklis}, ``Convergence speed in distributed consensus and averaging,'' {\em SIAM J. Control Optim., vol. 48, no. 1, pp. 33--55}, 2009.

\bibitem{koloskova2020unified}
A.~Koloskova, N.~Loizou, S.~Boreiri, M.~Jaggi, and S.~Stich, ``A unified theory of decentralized sgd with changing topology and local updates,'' in {\em Proc. Int. Conf. Mach. Learn.}, p.~5381–5393, PMLR, 2020.

\bibitem{boyd2011distributed}
{S. Boyd, N. Parikh, E. Chu, B. Peleato, J. Eckstein, et~al.}, ``Distributed optimization and statistical learning via the alternating direction method of multipliers,'' {\em Found. Trends in Mach. Learn.}, vol.~3, no.~1, pp.~1--122, 2011.

\bibitem{ryu2016primer}
{E. Ryu, S. P. Boyd}, ``Primer on monotone operator methods,'' {\em Appl. Comput. Math.}, vol.~15, no.~1, pp.~3--43, 2016.

\bibitem{dolev1993perfectly}
{D. Dolev, C. Dwork, O. Waarts, M. Yung}, ``Perfectly secure message transmission,'' {\em J. Assoc. Comput. Mach., vol. 40, no. 1, pp. 17-47,}, 1993.

\bibitem{dwork2006}
{C. Dwork}, ``Differential privacy,'' in {\em ICALP, pp. 1--12}, 2006.

\bibitem{dwork2006calibrating}
{C. Dwork, F. McSherry, K. Nissim, A. Smith}, ``Calibrating noise to sensitivity in private data analysis,'' in {\em Proc. Theory of Cryptography Conf. , pp. 265-284}, 2006.

\bibitem{cover2012elements}
{T. M. Cover and J. A. Tomas}, {\em Elements of information theory}.
\newblock John Wiley \& Sons, 2012.

\bibitem{Jane2020TIFS}
{Q. Li, J. S. Gundersen, R. Heusdens and M. G. Christensen}, ``Privacy-preserving distributed processing: Metrics, bounds, and algorithms,'' in {\em IEEE Trans. Inf. Forensics Secur.}, vol.~16, pp.~2090--2103, 2021.

\bibitem{lopuhaa2019information}
{ M. Lopuha{\"a}-Zwakenberg, B. {\v{S}}kori{\'c} and N. Li}, ``Information-theoretic metrics for local differential privacy protocols,'' {\em arXiv preprint arXiv:1910.07826}, 2019.

\bibitem{Jane2020GSP}
{Q. Li, M. Coutino, G. Leus and M. G. Christensen}, ``Privacy-preserving distributed graph filtering,'' in {\em Proc. Eur. Signal Process. Conf., pp. 2155-2159}, 2021.

\bibitem{yagli2020information}
{S. Yagli, A. Dytso and H.V. Poor}, ``Information-theoretic bounds on the generalization error and privacy leakage in federated learning,'' in {\em Proc. IEEE Int. Workshop Signal Process. Advances Wireless Commun.}, pp.~1--5, 2020.

\bibitem{bu2020tightening}
{Y. Bu, S. Zou, and V. V. Veeravalli}, ``Tightening mutual information-based bounds on generalization error,'' {\em IEEE J. Sel. Areas Info. Theory}, vol.~1, no.~1, pp.~121--130, 2020.

\bibitem{li2023adaptive}
M.~L.-Z. Q.~Li, J. S.~Gundersen and R.~Heusdens, ``Adaptive differentially quantized subspace perturbation (adqsp): A unified framework for privacy-preserving distributed average consensus,'' {\em IEEE Trans. Inf. Forensics Security.}, 2023.

\bibitem{liu2021quantitative}
Y.~Liu, X.~Zhu, J.~Wang, and J.~Xiao, ``A quantitative metric for privacy leakage in federated learning,'' in {\em Proc. Int. Conf. Acoust., Speech, Signal Process.}, pp.~3065--3069, IEEE, 2021.

\bibitem{mo2020layer}
F.~Mo, A.~Borovykh, M.~Malekzadeh, H.~Haddadi, and S.~Demetriou, ``Layer-wise characterization of latent information leakage in federated learning,'' {\em arXiv preprint arXiv:2010.08762}, 2020.

\bibitem{cuff2016differential}
{P. Cuff and L. Yu}, ``Differential privacy as a mutual information constraint,'' in {\em Proc. 23rd ACM SIGSAC Conf. Comput. Commun. Secur.}, pp.~43--54, 2016.

\bibitem{gotz2011publishing}
{ M. Gtz, A. Machanavajjhala, G. Wang, X. Xiao, J. Gehrke}, ``Publishing search logs—a comparative study of privacy guarantees,'' {\em IEEE Trans. Knowl. Data Eng.}, vol.~24, no.~3, pp.~520--532, 2011.

\bibitem{haeberlen2011differential}
{A. Haeberlen, B. C. Pierce, A. Narayan}, ``Differential privacy under fire.,'' in {\em Proc. 20th USENIX Conf. Secur.}, vol.~33, 2011.

\bibitem{xu2020subject}
X.~L. M.~Xu, ``Subject property inference attack in collaborative learning,'' in {\em Proc. 12th Int. Conf. Intell. Human Mach. Syst. Cybern.}, vol.~1, pp.~227--231, IEEE, 2020.

\bibitem{melis2019exploiting}
L.~Melis, C.~Song, E.~D. Cristofaro, and V.~Shmatikov, ``Exploiting unintended feature leakage in collaborative learning,'' in {\em Proc. IEEE Symp. Secur. Privacy}, pp.~691--706, IEEE, 2019.

\bibitem{wainakh2021user}
A.~Wainakh, F.~Ventola, T.~M{\"u}{\ss}ig, J.~Keim, C.~G. Cordero, E.~Zimmer, T.~Grube, K.~Kersting, and M.~M{\"u}hlh{\"a}user, ``User label leakage from gradients in federated learning,'' {\em arXiv preprint arXiv:2105.09369}, 2021.

\bibitem{shokri2017membership}
{R. Shokri, M. Stronati, C. Song, and V. Shmatikov}, ``Membership inference attacks against machine learning models,'' in {\em Proc. IEEE Symp. Secur. Privacy}, pp.~3--18, 2017.

\bibitem{he2019model}
Z.~He, T.~Zhang, and R.~Lee, ``Model inversion attacks against collaborative inference,'' in {\em 35th Annu. Comput. Secur. Appl. Conf.}, pp.~148--162, 2019.

\bibitem{wang2019beyond}
Z.~Wang, M.~Song, Z.~Zhang, Y.~Song, Q.~Wang, and H.~Qi, ``Beyond inferring class representatives: User-level privacy leakage from federated learning,'' in {\em IEEE Conf. Comput. Commun.}, pp.~2512--2520, IEEE, 2019.

\bibitem{yang2019neural}
Z.~Yang, J.~Zhang, E.~Chang, and Z.~Liang, ``Neural network inversion in adversarial setting via background knowledge alignment,'' in {\em 2019 ACM SIGSAC Conf. Comput. Commun. Secur.}, pp.~225--240, 2019.

\bibitem{zhang2020secret}
Y.~Zhang, R.~Jia, H.~Pei, W.~Wang, B.~Li, and D.~Song, ``The secret revealer: Generative model-inversion attacks against deep neural networks,'' in {\em IEEE/CVF Conf. Comput. Vision Pattern Recognit.}, pp.~253--261, 2020.

\bibitem{wang2004image}
Z.~Wang, A.~Bovik, H.~Sheikh, and E.~Simoncelli, ``Image quality assessment: from error visibility to structural similarity,'' {\em IEEE Trans. Image Process.}, vol.~13, no.~4, pp.~600--612, 2004.

\bibitem{li2022effective}
J.~Li, N.~Li, and B.~Ribeiro, ``Effective passive membership inference attacks in federated learning against overparameterized models,'' in {\em Proc. 11th Int. Conf. Learn. Representations}, 2022.

\bibitem{yeom2018privacy}
S.~Yeom, I.~Giacomelli, M.~Fredrikson, and S.~Jha, ``Privacy risk in machine learning: Analyzing the connection to overfitting,'' in {\em IEEE computer security foundations symposium (CSF)}, pp.~268--282, 2018.

\bibitem{song2021systematic}
L.~Song and P.~Mittal, ``Systematic evaluation of privacy risks of machine learning models,'' in {\em USENIX Security Symposium (USENIX)}, pp.~2615--2632, 2021.

\bibitem{jane2022elsevier}
{Q. Li, R. Heusdens and M. G. Christensen}, ``Communication efficient privacy-preserving distributed optimization using adaptive differential quantization,'' {\em Signal Process.}, 2022.

\bibitem{Jane2020TSP}
{Q. Li, R. Heusdens and M. G. Christensen}, ``Privacy-preserving distributed optimization via subspace perturbation: A general framework,'' in {\em IEEE Trans. Signal Process.}, vol.~68, pp.~5983 -- 5996, 2020.

\bibitem{jonkman2018quantisation}
{ J. A. G. Jonkman, T. Sherson, and R. Heusdens}, ``Quantisation effects in distributed optimisation,'' in {\em Proc. Int. Conf. Acoust., Speech, Signal Process.}, p.~3649–3653, 2018.

\bibitem{Jane2020ICASSP}
{Q. Li, R. Heusdens and M. G. Christensen}, ``Convex optimisation-based privacy-preserving distributed average consensus in wireless sensor networks,'' in {\em Proc. Int. Conf. Acoust., Speech, Signal Process.}, pp.~5895--5899, 2020.

\bibitem{dall2002random}
{J. Dall and M. Christensen}, ``Random geometric graphs,'' {\em Physical Rev. E}, vol.~66, no.~1, p.~016121, 2002.

\bibitem{jeon2021gradient}
J.~Jeon, K.~Lee, S.~Oh, J.~Ok, {\em et~al.}, ``Gradient inversion with generative image prior,'' {\em Adv. Neural Inf. Process. Syst.}, vol.~34, pp.~29898--29908, 2021.

\bibitem{lecun1998gradient}
Y.~LeCun, L.~Bottou, Y.~Bengio, and P.~Haffner, ``Gradient-based learning applied to document recognition,'' {\em Proc. IEEE}, vol.~86, no.~11, pp.~2278--2324, 1998.

\bibitem{deng2012mnist}
L.~Deng, ``The mnist database of handwritten digit images for machine learning research [best of the web],'' {\em IEEE Signal Process. Mag.}, vol.~29, no.~6, pp.~141--142, 2012.

\bibitem{cifar}
K.~Alex, H.~Geoffrey, {\em et~al.}, ``Learning multiple layers of features from tiny images,'' 2009.

\bibitem{connor2018function}
M.~O’Connor, G.~Zhang, W.~B. Kleijn, and T.~D. Abhayapala, ``Function splitting and quadratic approximation of the primal-dual method of multipliers for distributed optimization over graphs,'' {\em IEEE Trans. Signal Inf. Process. Netw.}, vol.~4, no.~4, p.~656–666, 2018.

\bibitem{simonyan2014very}
K.~Simonyan and A.~Zisserman, ``Very deep convolutional networks for large-scale image recognition,'' {\em arXiv preprint arXiv:1409.1556}, 2014.

\end{thebibliography}
\bibliographystyle{ieeetr}

\appendices
\section{Proof of Proposition ~\ref{prop.inter}}\label{pf.prop}
With \cref{eq.dZ},  the optimality condition \cref{eq.partialZero} can be expressed as $\forall i \in \mathcal{V}_h$:
\begin{align}\label{eq.parZeroPas}
&-(\tsum_{j \in \mathcal{N}_{i,c}} \v B_{i|j} \v z_{i|j}^{(t)} +\tsum_{j \in \mathcal{N}_{i,h}} \v B_{i|j} \tsum_{\tau=1}^{t}\Delta \v z_{i|j}^{(\tau)} +\rho d_i  \v w_i^{(t)})\nonumber \\
&= \nabla f_i(\v w_i^{(t)}) + \tsum_{j \in \mathcal{N}_{i,h}} \v B_{i|j} \v z_{i|j}^{(0)}.
\end{align}
Since all terms in the LHS of \cref{eq.parZeroPas} are known by the adversary, the noisy gradient (RHS of \cref{eq.parZeroPas}) is known by the adversary, proving claim $i)$. 
As for \cref{eq.graDif_i}, since the noise term the RHS of \cref{eq.parZeroPas} does not depend on $t$, the difference
\begin{align*} 
\nabla f(\v w_i^{(t+\ell)})-\nabla f(\v w_i^{(t)}),
\end{align*}
is known for any $\ell \geq 1$, hence proving claim $ii)$. Moreover,
by summing up the RHS of  \cref{eq.parZeroPas} over all honest nodes in $\mathcal{V}_{h,1}$, we have
\begin{align}
    &\tsum_{j\in \mathcal{V}_{h,1}} \left(\nabla f_j(\v w_j^{(t)}) + \tsum_{k \in \mathcal{N}_{j,h}} \v B_{j|k} \v z_{j|k}^{(0)}\right)\nonumber \\
    &=\tsum_{j\in \mathcal{V}_{h,1}} \nabla f_j(\v w_j^{(t)})+\tsum_{(j,k)\in \mathcal{E}_{h,1}}\v B_{j|k}( \v z_{j|k}^{(0)}-\v z_{k|j}^{(0)}) \label{eq.z0Partial},
\end{align}
since $\v B_{k|j} = -\v B_{j|k}$. 

The difference between two successive $\v z$ updates of \cref{eq.zupNQ} is given by
\begin{align}\label{eq:zt}
  \Delta {\v z}_{j|i}^{(t+1)}-\Delta {\v z}_{i|j}^{(t)}=2\rho \v B_{i|j} (\v w_i^{(t)}-\v w_i^{(t-1)}).
\end{align}
Hence, given the fact that at convergence we have $\v w_j^{(t_{\max})} = \v w_k^{(t_{\max})}$
for all $(j,k)\in {\cal E}$, the adversary has knowledge of all $\v w_j^{(t)}, j\in {\cal V}, t \in {\cal T}$. Moreover, again from \cref{eq.zupNQ}, we have
\begin{align}
 \Delta {\v z}_{j|i}^{(t+1)} &= \v z_{i|j}^{(t)} - \v z_{j|i}^{(t)} +    2\rho \v B_{i|j} \v w_i^{(t)} \nonumber \\
 &= \sum_{\tau=1}^t \left( \Delta {\v z}_{i|j}^{(\tau)} - \Delta {\v z}_{j|i}^{(\tau)}  \right) + \v z_{i|j}^{(0)} - \v z_{j|i}^{(0)},
 \label{eq:zdif}
\end{align}
showing that knowing the $\Delta {\v z}_{i|j}^{(t+1)} $s is equivalent to knowing the differences $\v z_{i|j}^{(0)} - \v z_{j|i}^{(0)}$ for all $(i,j)\in {\cal E}$.

Hence, the last term in the RHS of \cref{eq.z0Partial} is known to the adversary, the sum of gradients is thus known which completes the proof of claim $iii)$. 

\section{Proof of Theorem~\ref{thm_pas2}}\label{pf.thm_pas2}
We have
\begin{align}
&I( X_i; \{X_j\}_{j\in \mathcal{V}_c},\{ Z_{j \mid k}^{(0)}\}_{(j, k) \in \mathcal{E}_c},\{\Delta Z_{j \mid k}^{(t+1)}\}_{(j, k) \in \mathcal{E},t\in \mathcal{T}}) \nonumber\\
&\stackrel{(a)}{=}I\big
( X_i; \{X_j\}_{j\in \mathcal{V}_c},\{ Z_{j \mid k}^{(0)}\}_{(j, k) \in \mathcal{E}_c},\{W_j^{(t)},\big.\nonumber\\
  &\quad\big. Z_{j|k}^{(0)}-Z_{k|j}^{(0)}\}_{(j,k)\in \mathcal{E}, t\in \mathcal{T}}\big) \nonumber\\
&\stackrel{(b)}{=}I\big
( X_i; \{X_j\}_{j\in \mathcal{V}_c},\{ Z_{j \mid k}^{(0)}\}_{(j, k) \in \mathcal{E}_c},\{W_j^{(t)},
\big.\nonumber\\
  &\quad\big.\nabla f_j(W_j^{(t)}) + \tsum_{k \in \mathcal{N}_{j}} \v B_{j|k} Z_{j|k}^{(0)},Z_{j|k}^{(0)}-Z_{k|j}^{(0)}\}_{(j,k)\in \mathcal{E}, t\in \mathcal{T}}\big) \nonumber\\
  &\stackrel{(c)}{=}I\big( X_i; \{X_j\}_{j\in \mathcal{V}_c},\{ Z_{j \mid k}^{(0)}\}_{(j, k) \in \mathcal{E}_c},\{W_j^{(t)}\}_{j\in \mathcal{V}, t\in \mathcal{T}},\big.\nonumber\\
  &\quad\big.\{\nabla f_j(W_j^{(t)})\}_{j\in \mathcal{V}_c,t\in \mathcal{T}}, \{Z_{j|k}^{(0)}-Z_{k|j}^{(0)}\}_{(j,k)\in\mathcal{E}_h} ,\big.\nonumber\\
  &\quad\big.\{\nabla f_j(W_j^{(t)}) + \tsum_{k \in \mathcal{N}_{j,h}} \v B_{j|k} Z_{j|k}^{(0)}\}_{j\in \mathcal{V}_h,t\in \mathcal{T}}\big) \nonumber
\end{align}
where (a) follows from \cref{eq:zt} and \cref{eq:zdif}, and (b) follows from \cref{eq.parZeroPas}; (c) follows from \cref{eq.parZeroPas} and the fact that all terms in (b)   are sufficient to compute all terms in (c) and vice versa. Hence, the proof of \cref{eq.miDFLp_inter} is now complete.
Clearly, when $\sigma_Z^2=0$, the above (c) reduces to \cref{eq.miCFLp}, thereby showing that the privacy loss in DFL is upper bounded by the loss of CFL.

Before proving the lower bound, we first present the following lemma necessary for the coming proof.
\begin{lemma}\label{lm.linear}
    Let $X_1,\ldots X_n$ and  $R_1,\ldots R_n$ be independent random variables, and let $g(\cdot)$ be an arbitrary function. If they satisfy $\forall i\colon  I (X_i;g(X_i)+R_i)=0$. Then 
      \begin{align*}
   \forall i:  & I (X_i;g(X_1)+R_1,\ldots,g(X_n)+R_n,\tsum_{j=1}^{n} R_j)\\
   &=I(X_i;\tsum_{j=1}^{n} g(X_j) ).
  \end{align*}
\end{lemma}
\begin{proof}
See Appendix \ref{app.lemma}.
\end{proof}

When $\sigma^2_Z\rightarrow \infty$,  \cref{eq.miDFLp_inter} becomes
\begin{align*}
  &\stackrel{(a)}{=}I\big( X_i; \{X_j\}_{j\in \mathcal{V}_c},\{W_j^{(t)}\}_{j\in \mathcal{V}, t\in \mathcal{T}},\big.\nonumber\\
  &\quad\big.\{\nabla f_j(W_j^{(t)})\}_{j\in \mathcal{V}_c,t\in \mathcal{T}}, \{Z_{j|k}^{(0)}-Z_{k|j}^{(0)}\}_{(j,k)\in\mathcal{E}_h} ,\big.\nonumber\\
  &\quad\big.\{\nabla f_j(W_j^{(t)}) + \tsum_{k \in \mathcal{N}_{j,h}} \v B_{j|k} Z_{j|k}^{(0)}\}_{j\in \mathcal{V}_h,t\in \mathcal{T}}\big)\nonumber \\
    &\stackrel{(b)}{=}I\big( X_i; \{X_j\}_{j\in \mathcal{V}_c},\{W_j^{(t)}\}_{j\in \mathcal{V}, t\in \mathcal{T}},\{\nabla f_j(W_j^{(t)})\}_{j\in \mathcal{V}_c,t\in \mathcal{T}}, \big.\\
&\quad\big. \{\nabla f_j(W_j^{(t+1)})-\nabla f_j(W_j^{(t)})\}_{j\in \mathcal{V}_h,t\in \mathcal{T}}, \big.\nonumber\\
&\quad\big.\{\tsum_{j\in \mathcal{V}_{h,l}}\nabla f_j(W_j^{(t)})\}_{1\leq l \leq k_h, t\in \mathcal{T}}\big) 
\end{align*}
where (a)  uses the fact that $\sigma^2_Z\rightarrow \infty$, i.e., $\{ Z_{j \mid k}^{(0)}\}_{(j, k) \in \mathcal{E}_c}$ is asymptotically independent of all other terms; (b) uses Lemma \ref{lm.linear} and the fact that for all honest components $\mathcal{G}_{h,1}, \ldots, \mathcal{G}_{h,k_h}$ the last term in \cref{eq.z0Partial} is known: the gradient of honest nodes can be seen as the term $g(X_i)$'s  and $Z_{j|k}^{(0)}-Z_{k|j}^{(0)}$'s can be seen as noise $R_i$'s in  Lemma \ref{lm.linear}. Hence, proof of lower bound is complete.



\section{Proof of Lemma 1}\label{app.lemma}
We first present the following equality:
\begin{align} 
   & \begin{bmatrix}
        1 & 1 & \cdots & 1 & 1\\
        0 & 1 & \cdots & 1 & 1 \\
        \vdots & \ddots & \ddots & \vdots & \vdots\\
        0 & \cdots & 0 & 1 &1\\
        0 & \cdots & \cdots & 0 & 1
    \end{bmatrix}
    \begin{bmatrix}
        g(X_1)+R_1\\
        g(X_2)+R_2\\
        \vdots\\
        g(X_n)+R_n\\
        -\tsum_{i=1}^n R_i
    \end{bmatrix} \label{eq:lnMap}\\
    \quad&=
    \begin{bmatrix}
        \tsum_{i=1}^n g(X_i)\\
        \tsum_{i=2}^n g(X_i)- R_1\\
        \vdots\\
        X_n-\tsum_{i=1}^{n-1} R_i\\
        -\tsum_{i=1}^n R_i \nonumber
    \end{bmatrix} .
\end{align}
Thus   
 \begin{align*}
         & I (X_i;g(X_1)+R_1,\ldots,g(X_n)+R_n,\tsum_{j=1}^{n} R_j)\\
         &\stackrel{(a)} {=} I (X_i;\tsum_{j=1}^{n} g(X_j), \tsum_{j=2}^{n} g(X_j)-R_1, \ldots,g(X_n)\\
         &\quad -\tsum_{j=1}^{n-1} R_j,\tsum_{j=1}^{n} R_j)
         \\
         &=I(X_i;\tsum_{j=1}^{n} g(X_j) ).
    \end{align*}
where (a) follows from the fact that the linear map in \cref{eq:lnMap} is bijective; by observing the linear map we note that the difference of the $k$'th and $(k+1)$'th rows in RHS of \cref{eq:lnMap} is $g(X_k)+R_k$. This difference is independent of all $X_i$s and the $k$'th row of \cref{eq:lnMap}, thus  $X_i\rightarrow \tsum_{j=1}^{n} g(X_j)\rightarrow \tsum_{j=2}^{n} g(X_j)-R_1\rightarrow \ldots\rightarrow g(X_n)-\tsum_{j=1}^{n-1} R_i\rightarrow \tsum_{j=1}^{n} R_j$ forms a Markov chain. Which establishes the second equality, thereby completing the proof.

\section{Proof of Corollary \ref{cor.mi}}\label{pf.cor}
\begin{align*}
  &I(X_i;\mathcal{O}_{\rm {CFL}})-I(X_i;\mathcal{O}_{\rm {DFL}})\\
&\stackrel{(a)}{=}I(X_i;\mathcal{O}_{\rm {CFL}},\mathcal{O}_{\rm {DFL}})-I(X_i;\mathcal{O}_{\rm {DFL}})\\
&\stackrel{(b)}{=}I(X_i;\mathcal{O}_{\rm {CFL}},\mathcal{O}_{\rm {DFL}}|\mathcal{O}_{\rm {DFL}})\\
&=I(X_i;\mathcal{O}_{\rm {CFL}}|\mathcal{O}_{\rm {DFL}})\\
&\stackrel{(c)}{=}I\big( X_i;\{\nabla f_j(W_j^{(t)})\}_{j\in \mathcal{V}_h,t\in \mathcal{T}}|\{X_j\}_{j\in \mathcal{V}_c},\{W_j^{(t)}\}_{j\in \mathcal{V}, t\in \mathcal{T}},\big.\nonumber\\
&\quad\big.\{\nabla f_j(W_j^{(t)})\}_{j\in \mathcal{V}_c,t\in \mathcal{T}}, \{\nabla f_j(W_j^{(t+1)})-\nabla f_j(W_j^{(t)})\}_{j\in \mathcal{V}_h,t\in \mathcal{T}}, \big.\nonumber\\
&\quad\big.\{\tsum_{j\in \mathcal{V}_{h,l}}\nabla f_j(W_j^{(t)})\}_{1\leq l \leq k_h, t\in \mathcal{T}}\big) \\
&\stackrel{(d)}{=}I\big( X_i;\{\nabla f_j(W_j^{(t)})\}_{j\in \mathcal{V}_h,t\in \mathcal{T}}|\{X_j\}_{j\in \mathcal{V}_c},\{W_j^{(t)}\}_{j\in \mathcal{V}, t\in \mathcal{T}},\big.\nonumber\\
&\quad\big. \{\nabla f_j(W_j^{(t+1)})-\nabla f_j(W_j^{(t)})\}_{j\in \mathcal{V}_h,t\in \mathcal{T}}, \big.\nonumber\\
&\quad\big.\{\tsum_{j\in \mathcal{V}_{h,l}}\nabla f_j(W_j^{(t)})\}_{1\leq l \leq k_h, t\in \mathcal{T}}\big) 
\end{align*}
where (a) holds as $I(X_i;\mathcal{O}_{\rm {CFL}})=I(X_i;\mathcal{O}_{\rm {CFL}},\mathcal{O}_{\rm {DFL}})$ since all terms in \cref{eq.miCFLp} of CFL are sufficient to compute all terms in \cref{eq.miLower}
in DFL; (b) and (c) follow from the definition of conditional mutual information; (d) holds as $X_j$ and $W_j^{(t)}$ can determine $\nabla f_j(W_j^{(t)})$;

\section{Convergence behavior}\label{app.conv}
\subsection{Setting of logistic regression}
We assume each node $i$ holds $n_i=1$ data sample $\v x_{ik} \in \mathbb{R}^2$ and binary labels $\ell_{i} \in\{0,1\}$. For two labels, the input training samples are randomly drawn from a unit variance Gaussian distribution having mean ${\bm \mu}_0 = (-1,-1)^\intercal$ ($\ell_{ik}=0$) and mean ${\bm \mu}_1 = (1,1)^\intercal$ ($\ell_{ik}=1$), respectively.  We utilize PDMM for the decentralized protocol, i.e., $\theta=1$ in Algorithm \ref{alg:pdmm},  and the convergence parameter $\rho$ is set as $0.4$. A single-step 
gradient descent with learning rate $\mu=0.1$ is employed for updating \cref{eq.xupNQ}.   With this,  the bound \cref{equ:bound1_lr} becomes
\begin{align}
     &\frac{\partial f_i}{\partial \v w_{i}}^{\!\!\!(t+1)}-\frac{\partial f_i}{\partial \v w_{i}}^{\!\!\!(t)} = \tsum_{k=1}^{n_{i}}\big( \frac{1}{1+e^{-y^{(t+1)}_{ik}}} - \frac{1}{1+e^{-y^{(t)}_{ik}}}\big) \v x_{ik} \nonumber \\
     &= -\tsum_{j \in \mathcal{N}_i} \v B_{i|j} \Delta \v z_{w,i|j}^{(t)}+\frac{1}{\mu}\big(\v w_i^{(t)}-\v w_i^{(t+1)}\big)\nonumber\\
     &\quad+\big(\rho d_i-\frac{1}{\mu}\big)\big(\v w_i^{(t-1)}-\v w_i^{(t)}\big),
\end{align}
and similar modification follows for \cref{equ:bound2_lr}. As for the CFL protocol, the same learning rate is applied for a fair comparison, i.e., $\mu$ in \cref{eq.w_ave} is also set to $0.4$. 
\subsection{Setting and convergence behavior of DNNs}
To test the performance of the introduced DFL protocol, we generated a random geometric graph with $n=10$ nodes.  LeNet architecture~\cite{lecun1998gradient} is used for training and two datasets, the MNIST~\cite{deng2012mnist} and CIFAR-10 ~\cite{cifar} datasets, are used for evaluation.  MNIST and CIFAR-10 contain 60,000 images of size 28 $\times$ 28 and 50,000 images of size 32 $\times$ 32 from 10 classes, respectively. We randomly split each dataset into 10 folds, allocating each node one fold. For the decentralized protocol, we also use PDMM with the convergence parameter $\rho$ set to $0.4$. We use the quadratic approximation technique~\cite{connor2018function,niwa2020edge} to solve the sub-problems approximately with $\mu=\frac{1}{30}$. For the centralized protocol, the constant $\mu$ in \cref{eq.w_ave} is also set to $\mu=\frac{1}{30}$. In addition, we also test VGG-11 \cite{simonyan2014very} architecture of CIFAR-10 dataset with $n=8$ nodes, where we use inner iteration to solve the sub-problems and set the inner learning rate as $0.05$.
In \autoref{fig:effi} we demonstrate the training loss (in blue) and test accuracy (in red) for both protocols, applied on MNIST and CIFAR-10 datasets, respectively.

\begin{figure*}[ht]
    \centering
    \includegraphics[width=0.75\textwidth]{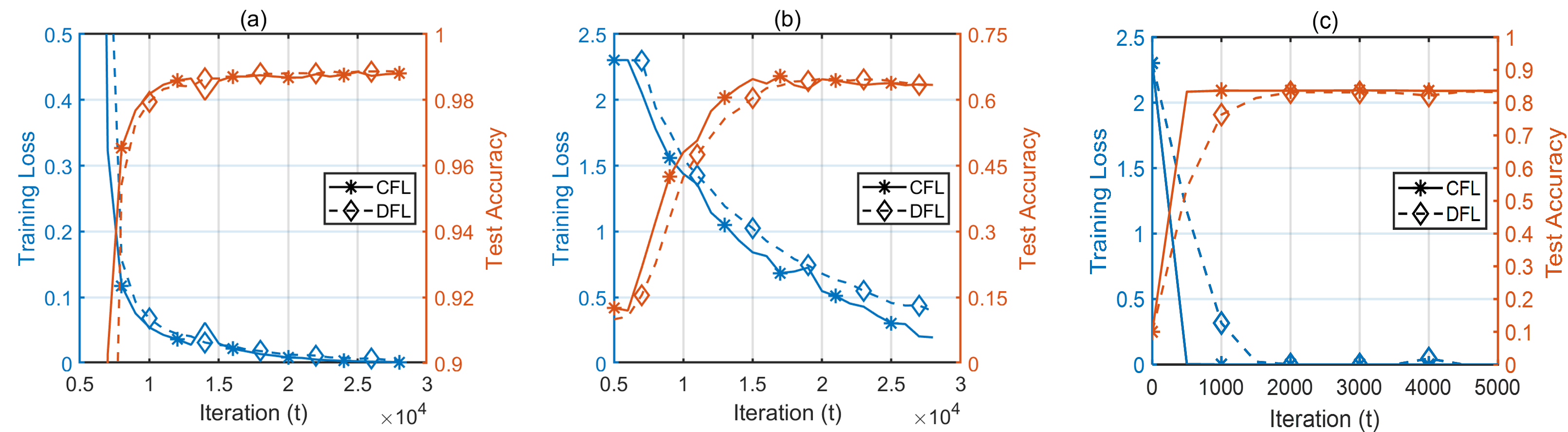}
    \vskip -8pt
    \caption{Training loss and test accuracy in terms of iteration number of both CFL and DFL for two datasets: (a) LeNet (MNIST); (b) LeNet (CIFAR-10); (c) VGG-11 (CIFAR-10).}
    \label{fig:effi}
\end{figure*}

\section{Supplementary attack results}\label{app.res}
\begin{figure*}[ht]
    \centering
    \includegraphics[width=0.8\textwidth]{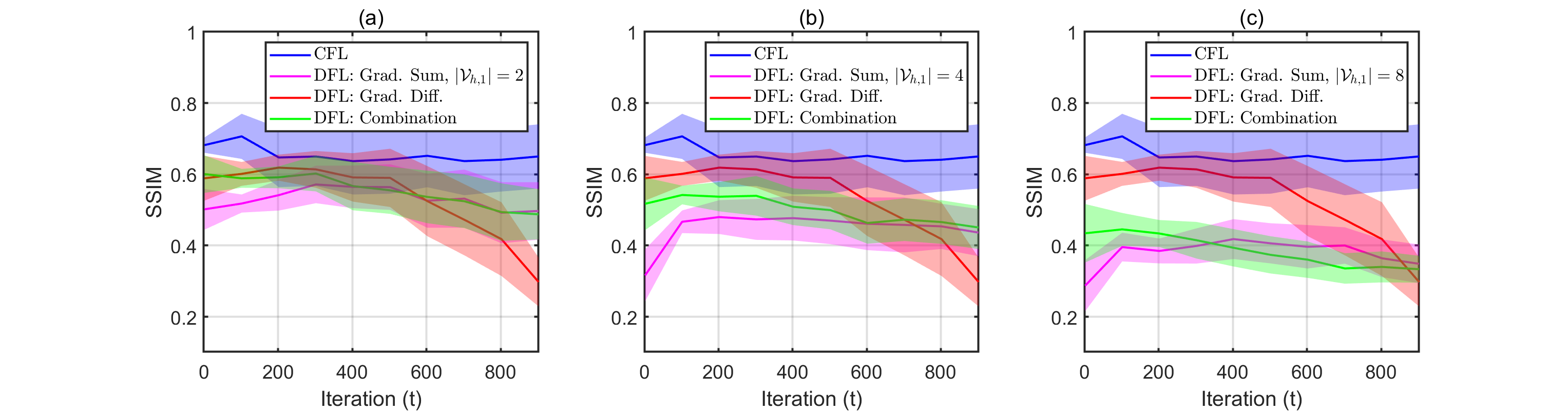}
    \vskip -6pt
    \caption{Performance comparisons of CFL and DFL via a combination strategy using MNIST dataset: SSIM of all reconstructed samples of CFL and DFL by combining gradient differences with the gradient sum for three different sizes of honest component $|\mathcal{V}_{h,1}|=2,4,8$ ((a)-(c), respectively).}
    \label{fig:comb}
\end{figure*}

\begin{figure*}[ht]
    \centering
\includegraphics[width=0.79\textwidth]{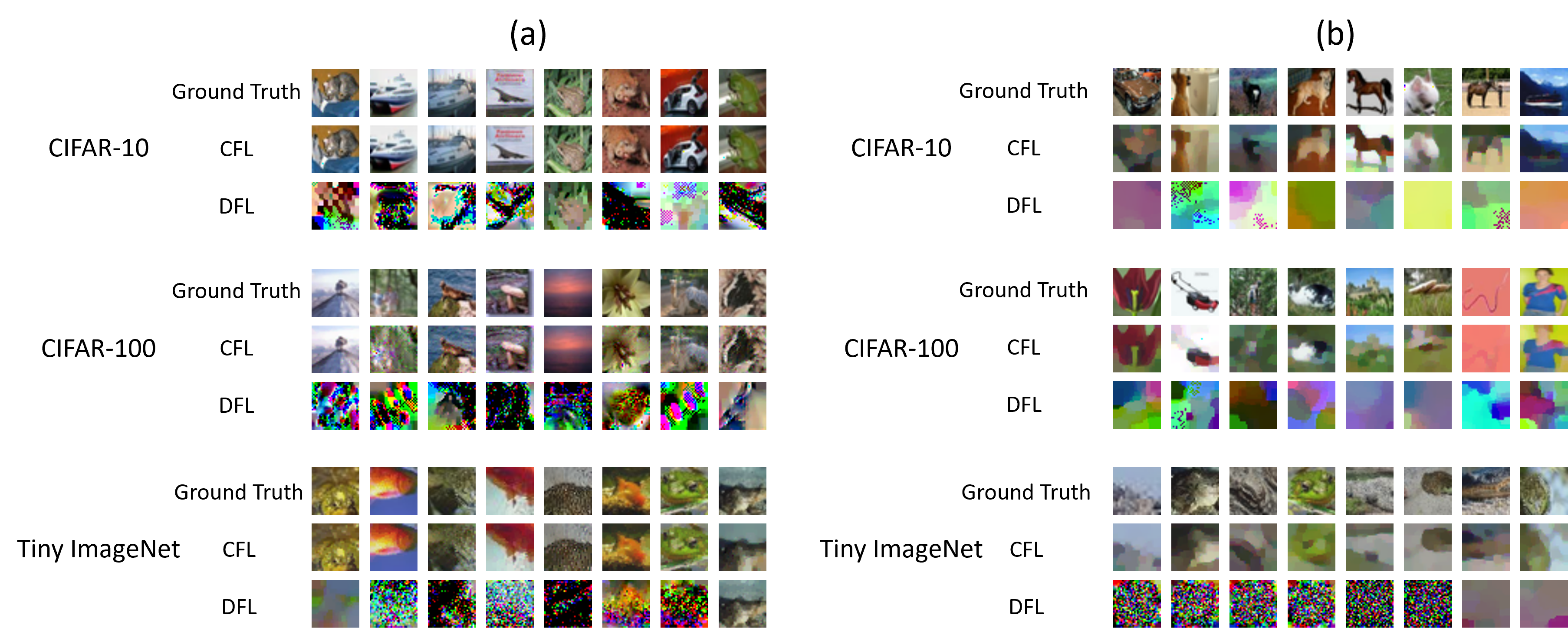}
    \vskip -6pt
    \caption{Reconstructed inputs in CFL and DFL using cosine similarity based gradient inversion attack \cite{geiping2020inverting} with (a)VGG-11 and (b)AlexNet architecture and three datasets (CIFAR-10, CIFAR-100 and tiny ImageNet). Label information in DFL is assumed to be known.}
    \label{fig:vgg}
\end{figure*}

\begin{figure*}[ht]
    \centering
    \includegraphics[width=0.8\textwidth]{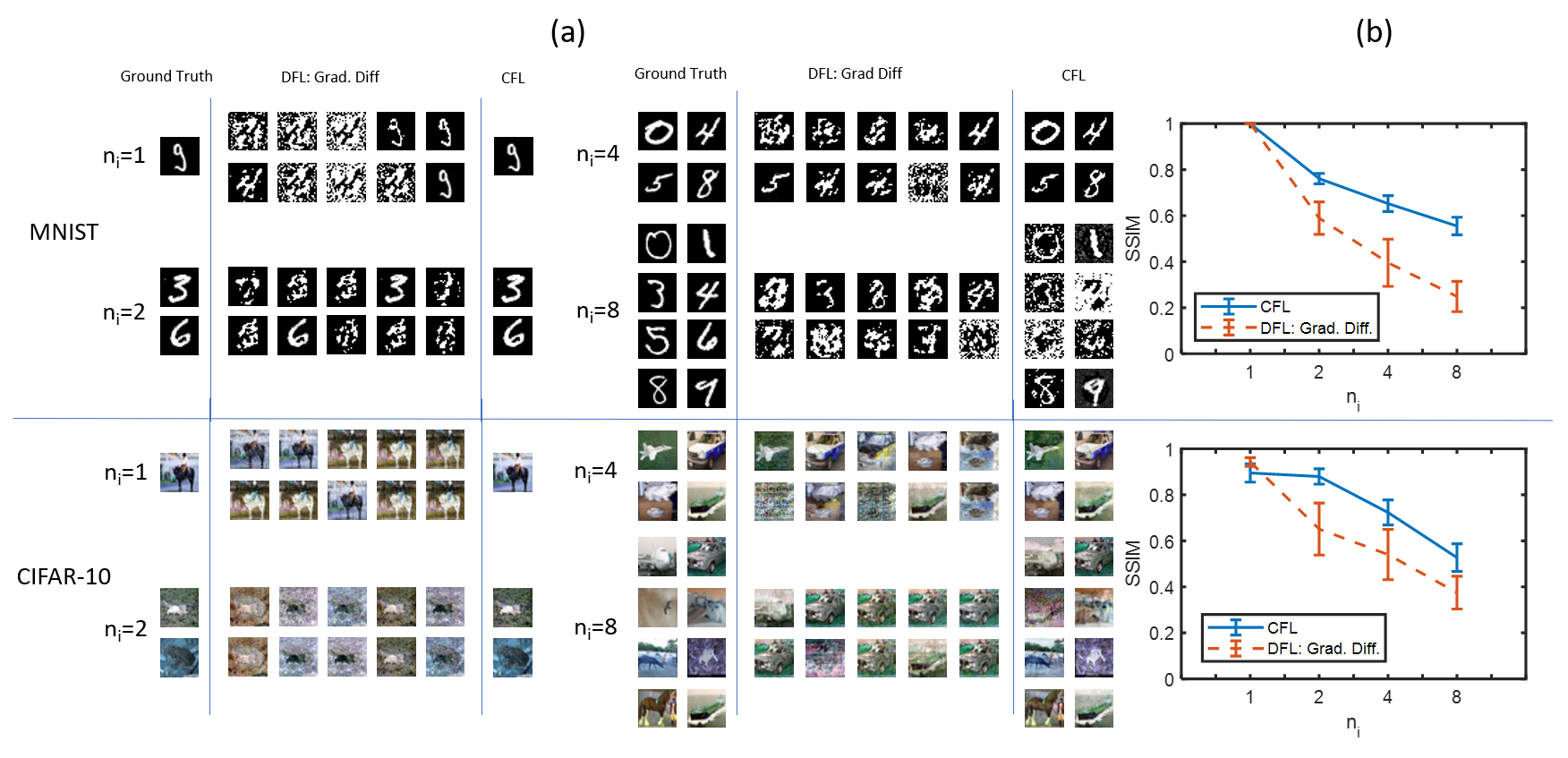}
    \caption{Performance comparisons of CFL and DFL via inverting inputs from gradient differences using cosine similarity based gradient inversion attack \cite{geiping2020inverting}: (a) Samples images of ground truth and reconstructed inputs, (b) SSIM comparisons of all reconstructed inputs for different batch size $n_i=1,2,4,8$, respectively.}
    \label{fig:graddiff_cos}
    \includegraphics[width=0.8\textwidth]{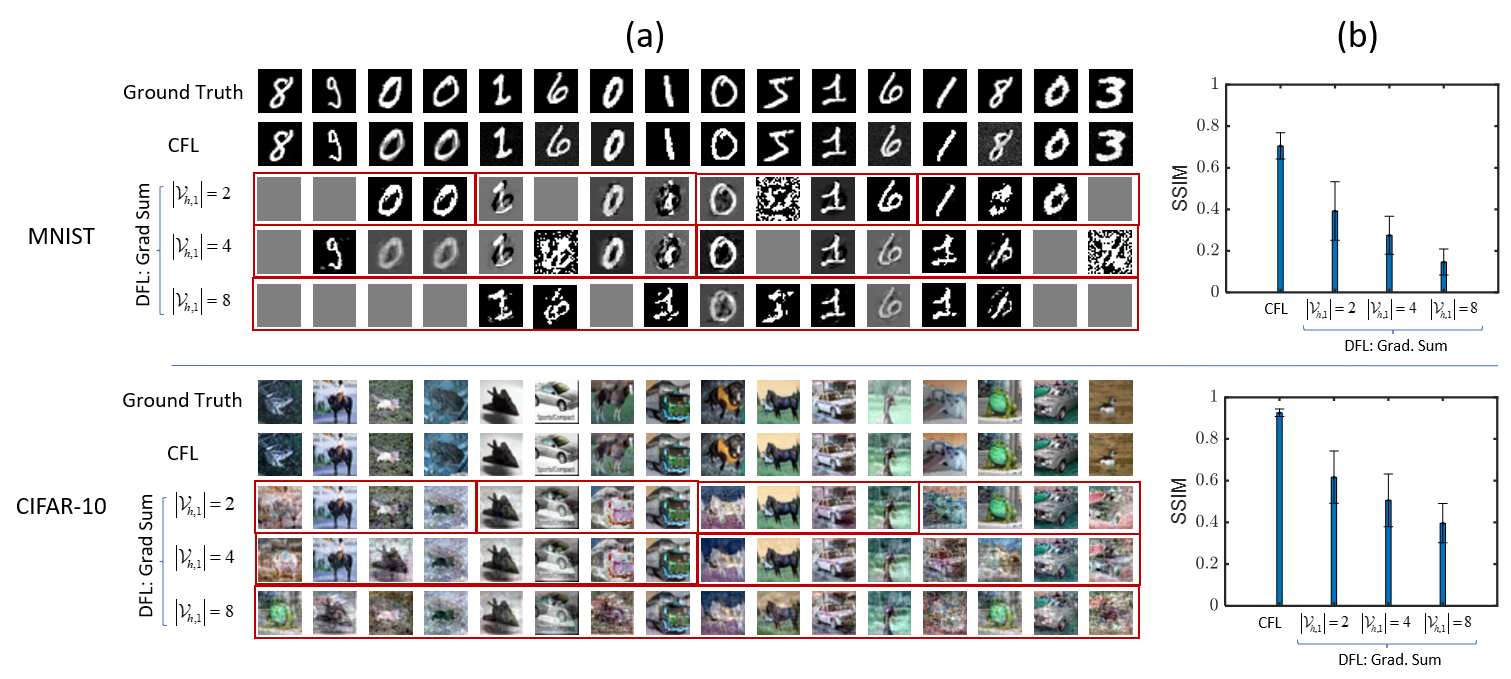}
    \caption{Performance comparisons of CFL and DFL via inverting inputs from gradient sums using cosine similarity based gradient inversion attack \cite{geiping2020inverting}: (a) Samples images of ground truth and reconstructed inputs, (b) SSIM of all reconstructed samples for three different sizes of honest component $|\mathcal{V}_{h,1}|=2,~4,~8$, respectively.}
    \label{fig:gradsum_cos}
\end{figure*}

\begin{figure*}[ht]
    \centering
\includegraphics[width=0.75\textwidth]{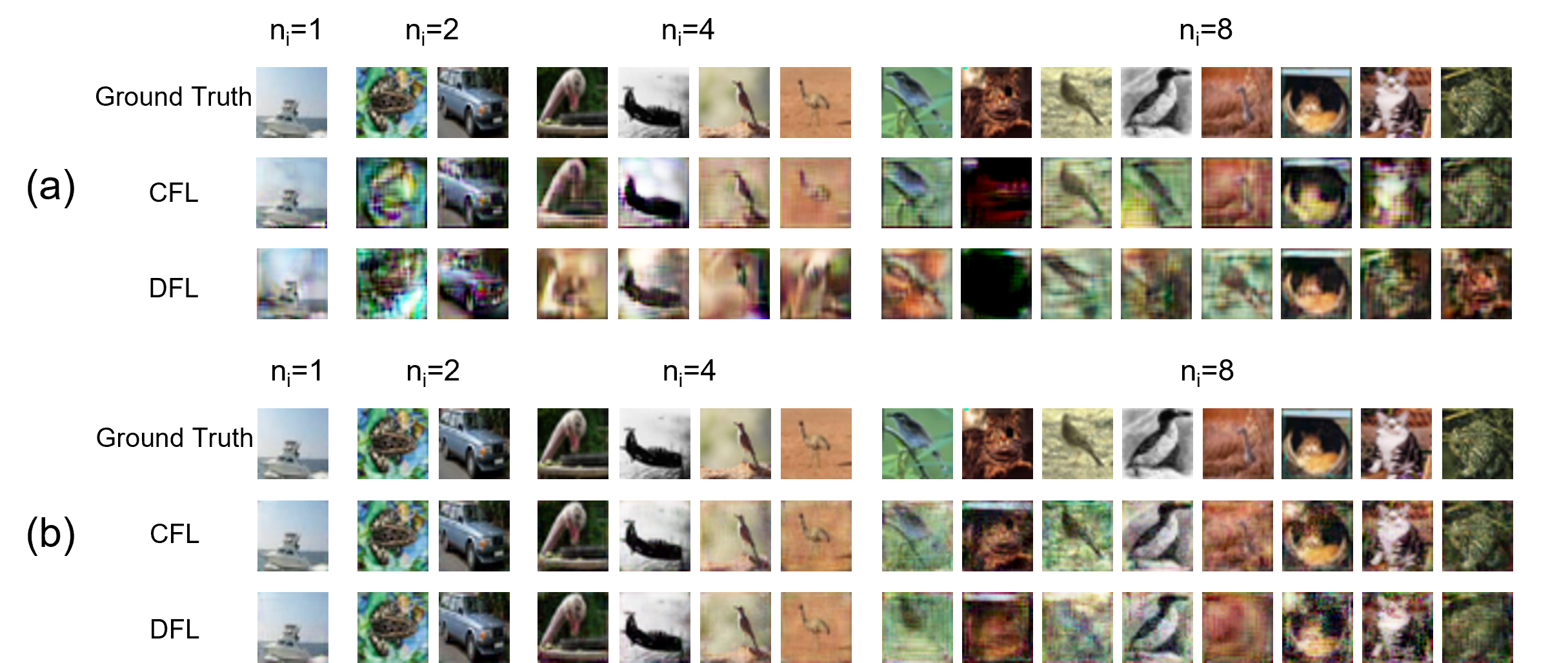}
    \vskip -6pt
    \caption{Reconstructed inputs of CIFAR-10 dataset in CFL and DFL via generative model-based gradient inversion attack \cite{jeon2021gradient} with (a)VGG-11  and (b)AlexNet. Label information in DFL is assumed to be known.}
    \label{fig:generative}
\end{figure*}

\noindent\textbf{Experimental setting of membership inference attack}:

Following similar settings in \cite{li2022effective}, we randomly select 4000 samples in Purchase dataset \cite{shokri2017membership} as the training sets and distribute them to $n=10$ nodes. For our experiments, we adopt the fully connected neural network architecture proposed by \cite{li2022effective}, comprising four perceptron layers: the first layer (FC1) contains 512 neurons, followed by the second layer (FC2) with 256 neurons, the third layer (FC3) with 128 neurons, and the final layer (FC4) with 100 neurons.

\end{document}